\documentclass[11pt]{article}
\usepackage[margin=1in]{geometry}
\usepackage{authblk}

\usepackage[utf8]{inputenc} 
\usepackage[T1]{fontenc}    
\usepackage{hyperref}       
\usepackage{url}            
\usepackage{booktabs}       
\usepackage{amsfonts}       
\usepackage{nicefrac}       
\usepackage{microtype}      
\usepackage{lipsum}
\usepackage{mathtools}
\usepackage{autonum}
\usepackage{cite}
\usepackage{amsmath}
\usepackage{amssymb}
\usepackage{amsthm}
\usepackage{algorithm,algcompatible}
\usepackage[toc,page]{appendix}

\usepackage{lineno}

\newtheorem{theorem}{Theorem}
\newtheorem{lemma}{Lemma}

\newtheorem{proposition}{Proposition}

\DeclareMathOperator*{\argmax}{arg\,max}
\DeclareMathOperator*{\argmin}{arg\,min}

\usepackage{mathptmx}      
\begin{document}

\title{Sharp Convergence Rates for Matching Pursuit
}






\author[1]{Jason M. Klusowski\thanks{jason.klusowski@princeton.edu}}
\author[2]{Jonathan W. Siegel\thanks{jwsiegel@tamu.edu}}
\affil[1]{\small Department of Operations Research and Financial Engineering, Princeton University}
\affil[2]{\small Department of Mathematics,
 Texas A\&M University, College Station}

 \date{}





\maketitle

\begin{abstract}
We study the fundamental limits of matching pursuit, or the pure greedy algorithm, for approximating a target function $ f $ by a linear combination $f_n$ of $n$ elements from a dictionary. When the target function is contained in the variation space corresponding to the dictionary, many impressive works over the past few decades have obtained upper and lower bounds on the error $\|f-f_n\|$ of matching pursuit, but they do not match. The main contribution of this paper is to close this gap and obtain a sharp characterization of the decay rate, $n^{-\alpha}$, of matching pursuit. Specifically, we construct a worst case dictionary which shows that the existing best upper bound cannot be significantly improved. It turns out that, unlike other greedy algorithm variants which converge at the optimal rate $ n^{-1/2}$, the convergence rate $n^{-\alpha}$ is suboptimal. Here, $\alpha \approx 0.182$ is determined by the solution to a certain non-linear equation.
\end{abstract}

\section{Introduction}

Matching pursuit \cite{mallat1993matching} is a widely used algorithm in signal processing that approximates a target signal by selecting a sparse linear combination of elements from a given dictionary.

Over the years, matching pursuit has garnered significant attention due to its effectiveness in capturing essential features of a signal with a parsimonious representation, offering reduced storage requirements, efficient signal reconstruction, and enhanced interpretability of the underlying signal structure. Because of this, its applications span various domains, including image, video, and audio processing and compression \cite{bergeaud1995matching, neff1997very}.


While previous works have explored the convergence properties of matching pursuit, several open questions and challenges remain. In particular, the relationship between the characteristics of the target signal, the chosen dictionary, and the convergence rate warrants further investigation. 
The main objective of this paper is to provide a comprehensive analysis of the convergence properties of matching pursuit. Understanding the convergence rate is crucial for assessing the algorithm's efficiency and determining the number of iterations required to achieve a desired level of approximation accuracy.

Let $H$ be a Hilbert space and $\mathbb{D}\subset H$ be a symmetric collection of unit vectors, i.e., $\|d\| = 1$ for $d\in \mathbb{D}$ and $d\in \mathbb{D}$ implies $-d\in \mathbb{D}$, called a dictionary. Non-linear dictionary approximation methods, which attempt to approximate a target function $f$ by a sparse linear combination
\begin{equation}\label{dictionary-expansion}
    f \approx \sum_{n=1}^s a_nd_n,
\end{equation}
where both the sequence of dictionary elements $d_i\in \mathbb{D}$ and the coefficients $a_i$ depend upon the function $f$ to be approximated, are common method in machine learning and signal processing. Such methods aim to generate an approximation of the form \eqref{dictionary-expansion} with a small number of terms $s$, and include gradient boosting \cite{friedman2001greedy}, $L^2$-boosting \cite{buhlmann2003boosting}, basis pursuit \cite{chen2001atomic}, and matching pursuit \cite{mallat1993matching}. 

In this work, we consider matching pursuit \cite{mallat1993matching}, which is a classical method for algorithmically generating a convergent non-linear dictionary expansion of the form 
\begin{equation}
    f = \sum_{n=1}^\infty a_nd_n.
\end{equation}
Matching pursuit is also known as the pure greedy algorithm \cite{devore1996some}, and is given by
\begin{equation}\label{pure-greedy-algorithm}
    f_0 = 0,~d_n \in \argmax_{d\in \mathbb{D}} \langle r_{n-1}, d\rangle,~f_n = f_{n-1} + \langle r_{n-1},d_n\rangle d_n,
\end{equation}
where $r_n = f - f_n$ is the residual at step $n$. We remark that the inner product $\langle\cdot,\cdot\rangle$ here is the inner product of the Hilbert space $H$, which is typically the $L_2$ inner product with respect to a probability distribution $\mu$ in practical applications. An equivalent way of writing this method, which explains the name pure greedy algorithm, is
\begin{equation}\label{pure-greedy-equivalent}
    (a_n,d_n) = \argmin_{a\in \mathbb{R}, d\in \mathbb{D}} \left\|f - \left(\sum_{i=1}^{n-1}a_id_i + ad\right)\right\| = \argmin_{a\in \mathbb{R}, d\in \mathbb{D}} \left\|r_{n-1} - ad\right\|.
\end{equation}
In each step we add the single term which minimizes the error the most, hence the name pure greedy algorithm. In other words, we fit a single term to the residual in each step.

An important generalization of this algorithm is the pure greedy algorithm with shrinkage $0 < s \leq 1$, see \cite[p. 375]{temlyakov2011greedy}, given by
\begin{equation}\label{pure-greedy-algorithm-shrinkage}
    f_0 = 0,~d_n \in \argmax_{d\in \mathbb{D}} \langle r_{n-1}, d\rangle,~f_n = f_{n-1} + s\langle r_{n-1},d_n\rangle d_n.
\end{equation}
Here we scale down the greedy term by a factor $s$, called the shrinkage factor, in each step.

Finding the optimal term $d_n$ in \eqref{pure-greedy-algorithm-shrinkage} could potentially involve solving a high-dimensional, non-convex optimization problem. To address this computational issue, gradient tree boosting \cite{friedman2001greedy} (i.e., $L^2$ boosting) features an additional step in which $\langle r_{n-1}, d\rangle$ is itself greedily optimized. Here the dictionary consists of normalized piecewise constant functions (or decision trees), which are fit to the residuals $r_{n-1}$ via the \emph{CART algorithm} \cite{breiman1984cart}. While CART is usually not motivated as a greedy way to optimize the inner product $\langle r_{n-1}, d\rangle$ over a collection of normalized piecewise constant functions, it can be equivalently formulated as such. Except for decision stumps (or depth-one trees) which involve only a 2-dimensional optimization, understanding the optimization gap for this greedy heuristic (CART) in general is currently an open problem, one that we do not address in the present paper. That is, we assume herein that the optimization problem is solved \emph{exactly} in \eqref{pure-greedy-algorithm-shrinkage}.

Despite the practical success and great interest generated by the method of matching pursuit, the precise convergence properties of the algorithm \eqref{pure-greedy-algorithm} for general dictionaries $\mathbb{D}$ have not been precisely determined. To describe this problem, we introduce the variation norm \cite{kurkova2001bounds,devore1998nonlinear} with respect the the dictionary $\mathbb{D}$, defined by
\begin{equation}
    \|f\|_{\mathcal{K}_1(\mathbb{D})} = \inf\{t > 0:~f\in tB_1(\mathbb{D})\},
\end{equation}
where the set $B_1(\mathbb{D})$ is the closed convex hull of $\mathbb{D}$, i.e.,
\begin{equation}
    B_1(\mathbb{D}) = \overline{\left\{\sum_{i=1}^N a_id_i,~a_i \geq 0,~\sum_{i=1}^N a_i = 1,~d_i\in \mathbb{D}\right\}}.
\end{equation}
This norm is a common measure of complexity when studying non-linear dictionary approximation, since it is well-defined and useful for any abstract dictionary $\mathbb{D}\subset H$ \cite{devore1998nonlinear,temlyakov2011greedy}. Results for the variation space $\mathcal{K}_1(\mathbb{D})$ can also often be extended to interpolation spaces which allow an even wider class of target functions to be analyzed \cite{barron2008approximation}. 

For some context, let us give a few examples of dictionaries $\mathbb{D}$ of interest and their corresponding variation spaces. For specific dictionaries of interest, the variation spaces have often been characterized and other descriptions are available. For example, the variation space $\mathcal{K}_1(\mathbb{D})$ for the dictionary
\begin{equation}
    \mathbb{D} = \{e^{(x-c)^2/(2\sigma^2)},~c\in \mathbb{R},~\sigma > 0\}
\end{equation}
of Gaussian bumps is equivalent to the Besov space the Besov space $B^1_{1,1}(\mathbb{R})$, which consists of function $f:\mathbb{R}\rightarrow \mathbb{R}$ satisfying
\begin{equation}
    \int_0^\infty t^{-1}\omega_2(f,t)_1\frac{dt}{t} < \infty,
\end{equation}
where $\omega_2(f,t)_1$ is the second order modulus of smoothness in $L_1$ given by
$$
    \omega_2(f,t)_1 = \sup_{0 < h < t}\int_\mathbb{R} |f(x+h) - 2f(x) + f(x-h)|dx.
$$
This equivalence was proved in \cite{meyer1992wavelets}. For more information on Besov spaces, we refer to \cite{devore1993constructive,devore1993besov}. 

Another example are the spaces $\mathcal{K}_1(\mathbb{D})$ for the dictionaries of ridge functions
\begin{equation}
    \mathbb{D} = \{\sigma(\omega\cdot x + b),~\omega\in S^{d-1},~b\in \mathbb{R}\},
\end{equation}
which correspond to shallow neural networks have been characterized and intensively studied \cite{ongie2019function,parhi2021banach,parhi2022kinds,siegel2022sharp,siegel2021characterization,klusowski2018approximation,ma2022barron}. Here $\sigma = \max(0,x)^k$ is the popular ReLU$^k$ activation function.

Typically the convergence of abstract greedy algorithms is studied on the variation space $\mathcal{K}_1(\mathbb{D})$ \cite{temlyakov2011greedy}. For instance, it is shown in \cite{devore1996some} that the pure greedy algorithm \eqref{pure-greedy-algorithm} satisfies
\begin{equation}
    \|f - f_n\| \leq \|f\|_{\mathcal{K}_1(\mathbb{D})}n^{-1/6}.
\end{equation}
This was subsequently improved by Konyagin and Temylakov \cite{konyagin1999rate} to
\begin{equation}
    \|f - f_n\| \leq C\|f\|_{\mathcal{K}_1(\mathbb{D})}n^{-11/62},
\end{equation}
and finally by Sil'nichenko \cite{sil2004rate} to
\begin{equation}\label{sil-estimate}
    \|f - f_n\| \leq C\|f\|_{\mathcal{K}_1(\mathbb{D})}n^{-\alpha},
\end{equation}
where $\alpha = \gamma/(2(2+\gamma)) \approx 0.182$ and $\gamma > 1$ is a root of the non-linear equation
\begin{equation}\label{gamma-equation}
    (1+\gamma)^{\frac{1}{2+\gamma}}\left(1+\frac{1}{1+\gamma}\right) - 1 - \frac{1}{\gamma} = 0.
\end{equation}
For the pure greedy algorithm with shrinkage \eqref{pure-greedy-algorithm-shrinkage} the method of Sil'nichenko implies that
\begin{equation}\label{sil-estimate-s}
    \|f - f_n\| \leq C\|f\|_{\mathcal{K}_1(\mathbb{D})}n^{-\alpha},
\end{equation}
where $\alpha = \gamma/(2(2+\gamma))$ and $\gamma > 1$ is a root of the non-linear equation
\begin{equation}\label{gamma-equation-s}
    (1+\gamma)^{\frac{1}{2+\gamma}}\left(1+\frac{1}{1+\gamma}\right) - 1 - \frac{2-s}{\gamma} = 0.
\end{equation}
As $s\rightarrow 0$, the exponent $\alpha \approx 0.305$ (c.f., the exponent 
$ 2/7 \approx 0.285 $ obtained in \cite{nelson2013greedy}). We remark that for $s = 0$, the pure greedy algorithm is stationary. This manifests itself in the fact that although the exponent $\alpha$ approaches $0.305$, the constant $C$ in the upper bound approaches $\infty$ as $s\rightarrow 0$. 

On the other hand, for functions $f\in \mathcal{K}_1(\mathbb{D})$ it is known \cite{pisier1981remarques,jones1992simple} that for each $n$ there exists an approximation
\begin{equation}
    f_n = \sum_{i=1}^n a_id_i,
\end{equation}
such that $\|f - f_n\| \leq \|f\|_{\mathcal{K}_1(\mathbb{D})}n^{-1/2}$, and that the exponent $n^{-1/2}$ is optimal for all dictionaries and $f\in \mathcal{K}_1(\mathbb{D})$ \cite{kurkova2001bounds}. 

Thus, in general the best we could hope for in the convergence of matching pursuit is a convergence rate like $n^{-1/2}$, which is attained by other greedy algorithms such as the orthogonal or relaxed greedy algorithm \cite{devore1996some,jones1992simple} (in fact, the orthogonal greedy algorithm has been shown to converge even faster for compact dictionaries \cite{siegel2022optimal,li2023entropy}). Remarkably, the convergence of matching pursuit is strictly worse. Specifically, it was shown in \cite{livshitz2003two} that there exists a dictionary $\mathbb{D}\subset H$ and an $f\in \mathcal{K}_1(\mathbb{D})$ such that the iterates of the pure greedy algorithm \eqref{pure-greedy-algorithm} satisfy
\begin{equation}
    \|f - f_n\| \geq C\|f\|_{\mathcal{K}_1(\mathbb{D})}n^{-0.27}.
\end{equation}
This estimate was finally improved in \cite{livshits2009lower} to
\begin{equation}\label{liv-lower-bound}
    \|f - f_n\| \geq C\|f\|_{\mathcal{K}_1(\mathbb{D})}n^{-0.1898}.
\end{equation}

Comparing the estimates \eqref{sil-estimate} and \eqref{liv-lower-bound} we see that there is a still a significant, though small, gap between the best upper and lower bounds. The goal of this work is to close this gap nearly completely. We show the following.
\begin{theorem}\label{main-theorem}
    Let $\gamma > 1$ be the root of the equation \eqref{gamma-equation}. Then for every $\alpha > \gamma/(2(2+\gamma))$, there exists a dictionary $\mathbb{D}\subset H$ and a function $f\in \mathcal{K}_1(\mathbb{D})$ such that the iterates of the pure greedy algorithm (i.e., matching pursuit) \eqref{pure-greedy-algorithm} satisfy
    \begin{equation}
        \|f - f_n\| \geq C\|f\|_{\mathcal{K}_1(\mathbb{D})}n^{-\alpha}.
    \end{equation}
\end{theorem}
Combined with the upper bound \eqref{sil-estimate}, this gives a precise characterization of the exponent in the convergence rate of matching pursuit 
and therefore makes significant progress towards solving an open problem posed in \cite{temlyakov2011greedy}, that is, to find the order of decay of the pure greedy algorithm. Additionally, this shows that the rate of convergence with shrinkage $s < 1$ is strictly better than for $s=1$, i.e., that \textit{any} amount of shrinkage improves the algorithm in the worst case, lending theoretical support to the empirical observation that some shrinkage in gradient tree boosting can result in significant performance improvements.

Finally, we remark that although the pure greedy algorithm is significantly worse than the relaxed or orthogonal greedy algorithm in the worst case over the variation space $\mathcal{K}_1(\mathbb{D})$, it achieves the same rate of convergence on some interpolation spaces between the variation space and the space $H$, see \cite{livshits2004rate} for details.

Although we do not give all of the details, our method shows that the rate obtained by Sil'nichenko \eqref{sil-estimate-s} is optimal also for shrinkage $s_0 < s < 1$ for a numerically computable value $s_0$. However, for sufficiently small shrinkage $s$ our method breaks down and it is an open problem to determine the rate of convergence as $s\rightarrow 0$.

\section{Construction of the Worst-Case Dictionary}
In this section, we give the proof of Theorem \ref{main-theorem}. This is based upon an extension and optimization of the constructions developed in \cite{livshitz2003two,livshits2009lower}.

 \subsection{Basic Construction}\label{basic-construction}
 We let $H = \ell^2$ and denote by $\{e_n\}_{n=1}^\infty$ its standard basis. Let $0 < \beta < \frac{1}{2}$. We will attempt to construct a realization of the pure greedy algorithm \eqref{pure-greedy-algorithm}, i.e., a dictionary $\mathbb{D}\subset H$ and an initial iterate $f\in \mathcal{K}_1(\mathbb{D})$, for which the convergence rate is bounded below by
 $$\|f - f_n\| \geq Cn^{-\frac{1}{2} + \beta}.$$ We will show that this construction succeeds as long as
 \begin{equation}\label{beta-condition}
  \left(\frac{\beta}{1-\beta}\right)^\beta\left(\frac{(1-\beta)^2}{1-2\beta}\right) < 1.
 \end{equation}
 Let us begin by verifying that this implies Theorem \ref{main-theorem}. It suffices to show that with equality in \eqref{beta-condition}, we have 
 \begin{equation}\label{eq-192}
 \alpha = \frac{1}{2} - \beta = \frac{\gamma}{(2(2+\gamma))},
 \end{equation}
 where $\gamma$ is the root of \eqref{gamma-equation}. Solving the relation \eqref{eq-192} for $\gamma$ in terms of $\beta$ we get
 \begin{equation}
     \gamma = \frac{1-2\beta}{\beta}.
 \end{equation}
 Plugging this into \eqref{gamma-equation}, we readily see that \eqref{gamma-equation} is equivalent to equality in \eqref{beta-condition}.
 
 Let us proceed with the lower bound construction. It will be convenient in what follows to work with the residual form of the pure greedy algorithm, given by
 \begin{equation}\label{residual-pure-greedy}
     r_0 = f,~d_n \in \argmin_{d\in \mathbb{D}} \langle r_{n-1},d\rangle,~r_n = r_{n-1} - \langle r_{n-1},d_n\rangle d_n.
 \end{equation}
 
 The construction, which is quite technical and follows the methods developed in \cite{livshitz2003two,livshits2009lower}, will depend upon four fundamental parameters in addition to $\beta$: an integer $K > 1$, an integer $N > 1$, a real number $\epsilon > 0$, and a function $\phi\in C^\infty([0,1])$.
 
Define two sequences of elements $\{r_n\}_{n=K-1}^\infty\subset H$ and $\{d_n\}_{n=K}^\infty \subset H$ inductively by
 \begin{equation}\label{inductive-definition-x}
  \begin{split}
   r_{K-1}&=-K^{-1/2+\beta}\frac{1}{\sqrt{K-1}}\sum_{i=1}^{K-1} e_i,\\
   d_n&=\gamma_nr_{n-1} + h_n + \xi_n e_n,\\
   r_n&=r_{n-1} - q_nd_n,
  \end{split}
 \end{equation}
 where $h_n\in H$ is defined by (recall that $\phi$ is one of our parameters)
\begin{equation}\label{definition-of-h}
 h_n = \sum_{i=1}^{n-1}\frac{\alpha_n}{n}\phi\left(\frac{i}{n}\right)e_i.
\end{equation}
Here $\gamma_n,\xi_n,\alpha_n$ and $q_n$ are numbers which are chosen so that the following three conditions are satisfied:
\begin{equation}\label{conditions-1}
 \|r_n\| = (n+1)^{-\frac{1}{2} + \beta},~q_n = \langle r_{n-1}, d_n\rangle,~\text{and}~\|d_n\| = 1.
\end{equation}
Then, we set (recall that $N$ and $\epsilon$ are parameters)
\begin{equation}\label{construction}
 \begin{split}
  f &= r_{N},\\
  \tilde{d}_N &= \epsilon \frac{r_{N}}{\|r_{N}\|} + \sqrt{1-\epsilon^2}d_{N},\\
  \mathbb{D} &= \{-\tilde{d}_N, \tilde{d}_N\}\cup\{-d_n, d_n\}_{n\geq N}.
 \end{split}
\end{equation}
 
From \eqref{inductive-definition-x} and \eqref{conditions-1}, we see that $\langle r_N,d_N\rangle = \langle r_{N-1} - q_Nd_N,d_N\rangle = 0$.
This implies that $\|\tilde{d}_N\| = 1$.
It is also clear that $f\in \text{span}(\tilde{d}_N,d_N)$, so that $\|f\|_{\mathcal{K}_1(\mathbb{D})} < \infty$. 

The idea is that the iteration in \eqref{inductive-definition-x} should follow the execution of the pure greedy algorithm applied to the dictionary $\mathbb{D}$ with initial residual $f$ (note that for notational convenience we begin indexing at $N$ instead of at $0$). In order to do this, we need to show that our parameters can be chosen appropriately so that the conditions in \eqref{conditions-1} can always be satisfied at every iteration, and so that the inner product inequalities
\begin{equation}\label{inner-prod-condition}
 |\langle r_{n-1}, \tilde{d}_N\rangle| < q_n~\text{and}~ |\langle r_{n-1}, d_k\rangle| < q_n
\end{equation}
hold for every $n > N$ and $k \geq N$ with $k\neq n$. This is divided into the following two technical results.
\begin{proposition}\label{first-part-main-proposition}
    Suppose that the smooth function $\phi:[0,1]\rightarrow \mathbb{R}_{\geq 0}$ satisfies the following conditions:
\begin{itemize}
    \item For some $0 < \delta < 1$, we have $\phi(x) = 0$ for $x\in [0,\delta]$.
    \item $\phi$ satisfies the following integral equality:
    \begin{equation}\label{phi-integral-condition-1}
  \left(\int_0^1 \phi(x)\left(1+\int_x^1 \phi(z)\frac{dz}{z}\right)dx\right) = \frac{\beta}{1 - 2\beta}.
 \end{equation}
 \item $\phi$ satisfies the following two integral inequalities:
 \begin{equation}\label{integral-intequality-1}
  \sup_{a\in[0,1]}\left|\int_0^a(\phi'(x)x - (\beta - 1)\phi(x)) \left(1 + \int_{a^{-1}x}^1\phi(z)\frac{dz}{z}\right)dx\right| < 1.
 \end{equation}
 \begin{equation}\label{integral-intequality-2}
  \sup_{a\in[0,1]}\left|\int_0^1\left((\beta - 1) + (\beta - 1)\int_{x}^1\phi(z)\frac{dz}{z} + \phi(x)\right)\left(\int_{ax}^x\phi(z)\frac{dz}{z}\right) dx + \int_a^1\phi(x)\frac{dx}{x}\right| < 1.
 \end{equation}
\end{itemize}
Then for sufficiently large $N,M > 1$ and sufficiently small $\epsilon > 0$, the conditions \eqref{conditions-1} can always be satisfied for every $n\geq K$ and the conditions \eqref{inner-prod-condition} can be satisfied for every $n > N$ and $k \geq N$ with $k\neq n$ (by making appropriate choices of $\gamma_n,\xi_n,\alpha_n$ and $q_n$). Thus, the iteration \eqref{inductive-definition-x} defines a realization of the pure greedy algorithm with dictionary $\mathbb{D}$ and initial residual $f$.
\end{proposition}
This Proposition is proved in Section \ref{first-part-main-proposition-proof-section} and borrows may of the ideas from \cite{livshits2009lower}. The method in \cite{livshits2009lower} can essentially be viewed as choosing $\phi$ to be a smoothed version of
\begin{equation}
    c\chi_{[\tau,1]}(x) := \begin{cases}
        c & x \geq \tau\\
        0 & x < \tau,
    \end{cases}
\end{equation}
where $c$ is chosen to satisfy \eqref{phi-integral-condition-1}, and then optimizing in $\tau$ and $\beta$. By generalizing to an arbitrary smooth function $\phi$ we are able to improve upon this analysis. That we can close the gap between existing upper and lower bounds and determine the sharp exponent in the rate of convergence this way is contained in the following Proposition.
\begin{proposition}\label{second-part-main-proposition}
    For any $0 < \beta < \frac{1}{2}$ satisfying the condition \eqref{beta-condition}, there exists a smooth function $\phi:[0,1]\rightarrow \mathbb{R}_{\geq 0}$ satisfying the conditions of Proposition \ref{first-part-main-proposition}.
\end{proposition}
This Proposition is proved in Section \ref{second-part-main-proposition-proof-section}. Combining these results, we obtain Theorem \ref{main-theorem}.

\section{Proof of Proposition \ref{first-part-main-proposition}}\label{first-part-main-proposition-proof-section}
We begin by determining conditions on the values $\gamma_n,\xi_n,\alpha_n$ and $q_n$ which are required to ensure that \eqref{conditions-1} is satisfied.
We calculate $q_n$ by noting that
\begin{equation}
 \|r_n\|^2 = \|r_{n-1}\|^2 - q_n^2,
\end{equation}
so that
\begin{equation}\label{q-definition}
 q_n^2 = \|r_{n-1}\|^2 - \|r_n\|^2 = n^{-1 + 2\beta} - (n+1)^{-1 + 2\beta}.
\end{equation}
We are free to choose $\gamma_n$, and following the argument in \cite{livshits2009lower} we make the choice $\gamma_n=(q_n^{-1} - q_{n-1}^{-1})$.
The conditions \eqref{conditions-1} and the construction \eqref{inductive-definition-x} imply that 
\begin{equation}
 q_n = \langle r_{n-1}, d_n\rangle = \gamma_n\|r_{n-1}\|^2 + \langle r_{n-1}, h_n\rangle = \gamma_nn^{-1+2\beta} + \langle r_{n-1}, h_n\rangle,
\end{equation}
which means that $\alpha_n$ must be chosen so that $h_n$ satisfies
\begin{equation}\label{condition-on-h}
 \langle r_{n-1}, h_n\rangle = q_n - \gamma_nn^{-1+2\beta}.
\end{equation}
Next, since $h_n\in \text{span}(e_1,...,e_{n-1})$, we see by induction that $d_n, r_n\in \text{span}(e_1,...,e_n)$. Using the construction \eqref{inductive-definition-x} this implies that
\begin{equation}\label{condition-on-xi}
 1 = \|d_n\|^2 = \|\gamma_nr_{n-1} + h_n\|^2 + \xi_n^2.
\end{equation}
This means that $h_n$ must also satisfy
\begin{equation}\label{condition-norm-hn}
    \|\gamma_nr_{n-1} + h_n\| \leq 1.
\end{equation}
Finally, we solve equation \eqref{condition-on-xi} for $\xi_n$ and choose $\xi_n = \sqrt{1 - \|\gamma_nr_{n-1} + h_n\|^2}$ to be the positive root.

\subsection{Inner Product Formulas}
Next, we derive formulas for the inner products appearing in \eqref{inner-prod-condition}, following \cite{livshits2009lower,livshitz2003two}. Consider first the case $k < n$. We will proceed by induction on $m > k$. Combining the last two equations in \eqref{inductive-definition-x} we see that
\begin{equation}\label{inductive-relationship-r-q}
 r_m = (1-q_m\gamma_m)r_{m-1} - q_mh_m - q_m\xi_me_m.
\end{equation}
Plugging the choice $\gamma_m = q_m^{-1} - q_{m-1}^{-1}$ into this, we get
\begin{equation}\label{eq-365}
 r_m = \frac{q_m}{q_{m-1}}r_{m-1} - q_mh_m - q_m\xi_me_m.
\end{equation}
Taking the inner product with $d_k$, noting that $d_k\in \text{span}(e_1,...,e_k)$, and that $m > k$ by assumption, we get
\begin{equation}\label{induction-equation-433}
 \langle r_m, d_k\rangle = \frac{q_m}{q_{m-1}}\langle r_{m-1}, d_k\rangle - q_m\langle h_m, d_k\rangle.
\end{equation}
Using \eqref{conditions-1} and the last line in \eqref{inductive-definition-x}, we see that by construction $\langle r_k, d_k\rangle = 0$. From this and equation \eqref{induction-equation-433} we see by induction on $m$ that
\begin{equation}\label{inner-product-rel-ngk}
 \langle r_{n-1}, d_k\rangle = -q_{n-1}\left(\sum_{i=k+1}^{n-1} \langle h_i,d_k\rangle\right),
\end{equation}
for $k < n$. 

The case of $\langle r_{n-1},\tilde{d}_N\rangle$ is a bit more complicated. We first note that since $\langle r_N,d_N\rangle = 0$ the formula \eqref{construction} implies that
\begin{equation}\label{initial-inner-product}
 \langle f, \tilde{d}_N\rangle = \epsilon\|r_N\|.
\end{equation}
Taking the inner product of \eqref{eq-365} with $\tilde{d}_N$, we see that 
\begin{equation}\label{inductive-relationship-d-N-tilde-inner}
    \langle r_m, \tilde{d}_N\rangle = \frac{q_m}{q_{m-1}}\langle r_{m-1}, \tilde{d}_N\rangle - q_m\langle h_m, \tilde{d}_N\rangle.
\end{equation}
Combining the base case \eqref{initial-inner-product} and the inductive step \eqref{inductive-relationship-d-N-tilde-inner}, we get
\begin{equation}\label{inner-product-rel-ntildeN}
 \langle r_{n-1}, \tilde{d}_N\rangle = -q_{n-1}\left(\sum_{i=N+1}^{n-1} \langle h_i,\tilde{d}_N\rangle - \epsilon\frac{\|r_N\|}{q_N}\right).
\end{equation}

Next, we consider the case $k > n$. In this case the relations \eqref{inductive-definition-x} give, for $k > n$
\begin{equation}\label{eq-382}
 \langle r_{n-1}, d_{k}\rangle = \gamma_k\langle r_{n-1}, r_{k-1}\rangle + \langle r_{n-1}, h_k\rangle,
\end{equation}
since $r_{n-1}\in \text{span}(e_1,...,e_{n-1})$. Further, we see that
\begin{equation}\label{eq-386}
 \langle r_{n-1}, r_{k-1}\rangle = \langle r_{n-1}, r_{k-2}\rangle - q_{k-1}\langle r_{n-1}, d_{k-1}\rangle.
\end{equation}
Shifting the indices in \eqref{eq-382} we also get
\begin{equation}
 \langle r_{n-1}, d_{k-1}\rangle = \gamma_{k-1}\langle r_{n-1}, r_{k-2}\rangle + \langle r_{n-1}, h_{k-1}\rangle.
\end{equation}
Solving this equation for $\langle r_{n-1}, r_{k-2}\rangle$ gives
\begin{equation}
    \langle r_{n-1}, r_{k-2}\rangle = \gamma_{k-1}^{-1}(\langle r_{n-1}, d_{k-1}\rangle - \langle r_{n-1}, h_{k-1}\rangle),
\end{equation}
and plugging that back into equation \eqref{eq-386} gives
\begin{equation}
    \langle r_{n-1}, r_{k-1}\rangle = (\gamma_{k-1}^{-1} - q_{k-1})\langle r_{n-1}, d_{k-1}\rangle - \gamma_{k-1}^{-1}\langle r_{n-1}, h_{k-1}\rangle.
\end{equation}
Finally, plugging this into \eqref{eq-382}, we get that
\begin{equation}\label{inductive-relation-kgn}
 \langle r_{n-1}, d_{k}\rangle = (\gamma_{k-1}^{-1} - q_{k-1})\gamma_k\langle r_{n-1}, d_{k-1}\rangle + \left\langle r_{n-1}, h_k - \frac{\gamma_k}{\gamma_{k-1}}h_{k-1}\right\rangle.
\end{equation}

\subsection{Parameter Asymptotics}
The estimates required to prove Proposition \ref{first-part-main-proposition} are quite technical and require the comparison of a variety of sums with corresponding integrals. For this purpose, we will need asymptotic formulas for some of the parameters introduced so far. These asymptotics were first derived in \cite{livshits2009lower,livshitz2003two}. 

We have from \eqref{q-definition} that
\begin{equation}
 q_n^2 = (1 - 2\beta)n^{-2+2\beta}(1 + c/n + O(n^{-2})),
\end{equation}
for a constant $c$ (the exact value of $c$ will not be important for us). This implies that
\begin{equation}\label{q-asymptotics}
 q_n = \sqrt{1 - 2\beta}n^{\beta - 1}(1 - 2\beta)n^{-2+2\beta}(1 + c'/n + O(n^{-2})),
\end{equation}
for a (different) constant $c'$.
From this, we obtain the following asymptotics for $\gamma_n$,
\begin{equation}\label{gamma-asymptotics}
 \gamma_n = (q_n^{-1} - q_{n-1}^{-1}) = \frac{(1-\beta)}{\sqrt{1 - 2\beta}}n^{-\beta}(1 + O(n^{-1})).
\end{equation}
These two formulas give us the asymptotics of the right hand side of \eqref{condition-on-h}
\begin{equation}\label{inner-product-h-asymptotics}
 q_n - \gamma_nn^{-1+2\beta} = \left(\sqrt{1 - 2\beta} - \frac{(1-\beta)}{\sqrt{1 - 2\beta}}\right)n^{\beta - 1}(1 + O(n^{-1}))
 = \frac{-\beta}{\sqrt{1 - 2\beta}}n^{\beta - 1}(1 + O(n^{-1})).
\end{equation}

Finally, we will also need the asymptotics for the inductive factor in \eqref{inductive-relation-kgn}
\begin{equation}
 (\gamma_{m-1}^{-1} - q_{m-1})\gamma_m = \frac{\gamma_m}{\gamma_{m-1}}(1-\gamma_{m-1}q_{m-1}) = \frac{\gamma_m}{\gamma_{m-1}}\frac{q_{m-1}}{q_{m-2}}.
\end{equation}
The asymptotics for $\gamma_m$ \eqref{gamma-asymptotics} and the asymptotics for $q_n$ \eqref{q-asymptotics} imply that
\begin{equation}\label{quotient-asymptotics}
 \frac{\gamma_m}{\gamma_{m-1}} = \left(1 - \frac{\beta}{m} + o(m^{-1})\right),~\frac{q_{m-1}}{q_{m-2}} = \left(1 + \frac{\beta - 1}{m} + o(m^{-1})\right).
\end{equation}
Multiplying these, we get
\begin{equation}\label{inductive-factor-asymptotics}
 (\gamma_{m-1}^{-1} - q_{m-1})\gamma_m = (1 - (1 + o_m(1))m^{-1}).
\end{equation}
Note that here and in the following, we use the notation $o_k(1)$ to denote a quantity depending upon the parameter $k$ whose limit is $0$ as $k\rightarrow\infty$. 
We also adopt the convention that the constants in big-$O$ expressions and the rate of decay to $0$ in little-$o$ expressions will be uniform in any lower case indices $n,m,k,$ or $l$ which do not appear in the argument (or the subscript if the argument is $1$), but may depend upon other parameters of the construction. This modified big-$O$ notation is necessary to simplify the comparisons between sums and integrals in the following.

Finally, we recall the basic fact, which we will use without reference in the following, that (using this convention)
\begin{equation}
 \frac{1}{n}\sum_{k=1}^n o_k(1) = o_n(1).
\end{equation}

\subsection{Estimates for the Conditions \eqref{conditions-1}}



In this Section, we show that for sufficiently large $K$, the parameters $\xi_n$ and $\alpha_n$ can be chosen to satisfy the conditions \eqref{conditions-1}. For this, we will need the following formula for the coefficients of $r_n$, which will also be useful later.

\begin{lemma}\label{components-of-r-formula}
 The following formula holds for $K \leq k \leq n$
 \begin{equation}\label{large-components-of-r}
  \langle r_n, e_k\rangle = -q_n\left(\xi_k + \sum_{j=k+1}^n \langle h_j, e_k\rangle\right).
 \end{equation}
 Furthermore, for $0 < k \leq K-1\leq n$, we have
 \begin{equation}\label{small-components-of-r}
  \langle r_n, e_k\rangle = -q_n\left(\frac{K^{-1/2+\beta}}{q_{K-1}\sqrt{K-1}} + \sum_{j=K}^n \langle h_j, e_k\rangle\right).
 \end{equation}
\end{lemma}
Since we also have $\langle r_n, e_k\rangle = 0$ for $k > n$, this lemma gives the components of all of the iterates $r_n$.
\begin{proof}
 We prove this formula for any fixed $k > 0$ by induction on $n$, with the base case being $n = \max(K-1,k)$. The base case follows in the case where $n = k\geq K$ from
 \begin{equation}
  \langle r_k,e_k\rangle = \langle r_{k-1} - q_k d_k, e_k\rangle = -q_k\langle d_k, e_k\rangle = -q_k \xi_k,
 \end{equation}
 since $r_{k-1}, h_k\in \text{span}(e_1,...,e_{n-1})$. When $k \leq K-1$, we use the definition of $r_{K-1}$ to get \eqref{small-components-of-r} in the case $n = K-1$.
 
 The inductive step follows from \eqref{eq-365} since
 \begin{equation}
  \langle r_n, e_k\rangle = \frac{q_n}{q_{n-1}}\langle r_{n-1}, e_k\rangle - q_n\langle h_n, e_k\rangle.
 \end{equation}
Using the inductive hypothesis, this gives
\begin{equation}
 \langle r_n, e_k\rangle = -\frac{q_n}{q_{n-1}}q_{n-1}\left(\xi_k + \sum_{j=k+1}^{n-1} \langle h_j, e_k\rangle\right) - q_n\langle h_n, e_k\rangle = -q_n\left(\xi_k + \sum_{j=k+1}^n \langle h_j, e_k\rangle\right),
\end{equation}
when $k \geq K$, and similarly when $k\leq K-1$.

\end{proof}

Using this, we prove the following.
\begin{proposition}\label{large-K-lemma}
 Let $\phi:[0,1]\rightarrow \mathbb{R}_{\geq 0}$ be a non-zero $C^\infty$ function. Then for sufficiently large $K$, equations \eqref{definition-of-h} and \eqref{condition-on-h} are solvable for $\alpha_n$. In addition, $\alpha_n$ is non-negative and the resulting value of $h_n$ satisfies \eqref{condition-norm-hn}, so that $\xi_n$ is defined via \eqref{condition-on-xi}. Finally, the construction satisfies
 \begin{equation}
  \lim_{n\rightarrow \infty} \xi_n = 1.
 \end{equation}
\end{proposition}
\begin{proof}
 Note that since $\phi\neq 0$ is non-negative, we have
 \begin{equation}
  \int_0^1 \phi(x)dx = C_\phi > 0.
 \end{equation}

 We begin by calculating the inner product $\langle r_{n-1},h_n\rangle$ for $n \geq K$ using Lemma \ref{components-of-r-formula} and the definition of $h_n$ in \eqref{definition-of-h}. We get
 \begin{equation}\label{residual-h-inner-formula}
  \begin{split}
  \langle r_{n-1},h_n\rangle = \sum_{k=1}^{n-1}\langle r_{n-1},e_k\rangle\langle h_n,e_k\rangle = &-\frac{q_{n-1}\alpha_n}{n}\left(\sum_{k=1}^{n-1}\sum_{j=\max(k+1,K)}^{n-1}\frac{\alpha_j}{j}\phi\left(\frac{k}{n}\right)\phi\left(\frac{k}{j}\right) \right.\\ &+ \left.\frac{K^{-\frac{1}{2} + \beta}}{q_{K-1}\sqrt{K-1}}\sum_{k=1}^{K-1}\phi\left(\frac{k}{n}\right) + \sum_{k=K}^{n-1}\xi_k\phi\left(\frac{k}{n}\right) \right).\\
  &
  \end{split}
 \end{equation}
  For sufficiently large $K$, we now show by induction on $n \geq K$ that  we can solve \eqref{definition-of-h} and \eqref{condition-on-h} for $\alpha_n \geq 0$, and that the resulting $h_n$ will satisfy \eqref{condition-on-h} and $\|\gamma_nr_{n-1} + h_n\| \leq \sqrt{3}/2$, which implies that for the $\xi_n$ satisfying \eqref{condition-on-xi} we will have $\xi_n \geq 1/2$.

So assume inductively that for all $K \leq k < n$ we have $\xi_k \geq \frac{1}{2}$ and $\alpha_k \geq 0$. (Note that in the base case $n=K$ there is nothing to assume.) Since $\phi \geq 0$, we use the inductive assumption on $\alpha_k$ to remove the initial double sum on the right hand side of \eqref{residual-h-inner-formula} and get the inequality
 \begin{equation}\label{eq-495}
  \langle r_{n-1},h_n\rangle \leq -\frac{q_{n-1}\alpha_n}{n}\left(\frac{K^{-\frac{1}{2} + \beta}}{q_{K-1}\sqrt{K-1}}\sum_{k=1}^{K-1}\phi\left(\frac{k}{n}\right) + \sum_{k=K}^{n-1}\xi_k\phi\left(\frac{k}{n}\right)\right).
 \end{equation}
 Using the asymptotics \eqref{q-asymptotics}, \eqref{gamma-asymptotics} and \eqref{inner-product-h-asymptotics}, and comparing the Riemann sum in the previous equation with an integral, a straightforward calculation shows that $K$ can be chosen sufficiently large so that $0 \leq \alpha_n \leq C$ (for an appropriate constant $C$) and 
 \begin{equation}
     \|\gamma_nr_{n-1} + h_n\| \leq \gamma_n\|r_{n-1}\| + \|h_n\| \leq \frac{\sqrt{3}}{2},
 \end{equation}
 which implies that $\xi_n \geq \frac{1}{2}$. Moreover, since $\alpha_n\leq C$ we readily see that
 \begin{equation}
     \|\gamma_nr_{n-1} + h_n\| \leq \gamma_n\|r_{n-1}\| + \|h_n\| \leq \gamma_n\|r_{n-1}\| + \alpha_n\left[\sup_{x\in [0,1]}\phi(x)\right]n^{-\frac{1}{2}}\rightarrow 0,
 \end{equation}
 as $n\rightarrow \infty$, which implies that $\lim_{n\rightarrow \infty} \xi_n = 1$.

\end{proof}

\subsection{Estimates for the Inner Product Conditions}
In the following analysis, we assume that $K$ has been chosen sufficiently large so that the conclusion of Proposition \ref{large-K-lemma} holds, and turn to proving that the conditions \eqref{inner-prod-condition} are satisfied under the assumptions of Proposition \ref{first-part-main-proposition} for sufficiently large $N$ and small $\epsilon$.

For this, we will need the following estimate relating sums to integrals of $\phi$.
\begin{lemma}\label{phi-integral-estimates}
 Suppose that $\phi:[0,1]\rightarrow \mathbb{R}_{\geq 0}$ is a $C^\infty$ function such that for some $\delta > 0$, we have $\phi(x) = 0$ on $[0,\delta]$. Let $n \geq l \geq k$, then we have
 \begin{equation}
  \sum_{j=l}^{n-1} \frac{1}{j}\phi\left(\frac{k}{j}\right) = \int_{\frac{k}{n}}^{\frac{k}{l}} \phi(x) \frac{dx}{x} + O(k^{-1}).
 \end{equation}
 Here the constant in $O(k^{-1})$ only depends upon $\phi$ and not upon $n$ and $l$.
\end{lemma}
\begin{proof}
 Consider the sequence of $n-l$ points 
 \begin{equation}
  x_0 := \frac{k}{n} < x_1 := \frac{k}{n-1} < \cdots < x_{n-l-1} := \frac{k}{l+1} < x_{n-l} := \frac{k}{l}.
 \end{equation}
 This sequence of points forms a partition $\Delta$ of the interval $[\frac{k}{n},\frac{k}{l}]$ and the gap (or mesh/norm) of $\Delta$ satisfies
 \begin{equation}
  |\Delta| := \sup_i |x_{i+1} - x_i| \leq \frac{k}{l} - \frac{k}{l+1} = \frac{k}{l(l+1)} < \frac{1}{k}.
 \end{equation}
 Moreover, we have the following estimate
 \begin{equation}
  \log\left(\frac{k}{j}\right) - \log\left(\frac{k}{j + 1}\right) = \log\left(\frac{j+1}{j}\right) = \log\left(1 + \frac{1}{j}\right) = \frac{1}{j} + O\left(\frac{1}{j^2}\right).
 \end{equation}
 Since $\phi$ is bounded, we obtain
 \begin{equation}
 \begin{split}
  \sum_{j=l}^{n-1} \frac{1}{j}\phi\left(\frac{k}{j}\right) &= \sum_{j=1}^{n-l} \phi\left(\frac{k}{j}\right)\left[\log\left(\frac{k}{j}\right) - \log\left(\frac{k}{j + 1}\right)\right] + O\left(\sum_{j=l}^{n-1} \frac{1}{j^2}\right) \\
  &= \sum_{i=1}^{n-l} \phi(x_i)[\log(x_i) - \log(x_{i-1})] + O(k^{-1}).
  \end{split}
 \end{equation}
 Finally, since $\phi$ is smooth and $\log(x)$ is increasing and bounded on $[\delta,1]$, we obtain, by comparing with the Riemann-Stieltjes integral and using that the mesh $\Delta$ satisfies $|\Delta| < k^{-1}$,
 \begin{equation}
  \sum_{i=1}^{n-l} \phi(x_i)[\log(x_i) - \log(x_{i-1})] = \int_{\frac{k}{n}}^{\frac{k}{l}} \phi(x) d\log(x) + O(k^{-1}) = \int_{\frac{k}{n}}^{\frac{k}{l}} \phi(x) \frac{dx}{x} + O(k^{-1}).
 \end{equation}

\end{proof}

Similar to the argument given in \cite{livshits2009lower}, the next step is to carefully study the sequence $\alpha_n$. We also give an asymptotic formula for the coefficients of $r_n$.
\begin{proposition}\label{bar-alpha-lemma}
 Suppose that $\phi$ satisfies the condition \eqref{phi-integral-condition-1} in Proposition \ref{first-part-main-proposition}.
 Then, in the preceding construction, with $K$ chosen large enough such that the conclusion of Proposition \ref{large-K-lemma} holds, the sequence $\alpha_n$ in \eqref{definition-of-h} will satisfy
 \begin{equation}\label{alpha-limit-equation}
  \lim_{n\rightarrow \infty} \alpha_n = 1.
 \end{equation}
 More precisely, this convergence will be such that
 \begin{equation}\label{alpha-asymptotics-equation}
  \lim_{n\rightarrow \infty} \frac{\alpha_{n+1}}{\alpha_n} = 1 + o(n^{-1}).
 \end{equation}
 In addition, under these conditions as $k\rightarrow \infty$ we have the following asymptotics for the coefficients of $r_n$
 \begin{equation}\label{asymptotics-coefficients-of-r}
  \langle r_n, e_k\rangle = -q_n\left(1 + \int_\frac{k}{n}^1 \phi(z)\frac{dz}{z} + o_k(1)\right),
 \end{equation}
 for any $n \geq k$. (Recall that by our convention $o_k(1)$ goes to $0$ as $k\rightarrow \infty$ uniformly in $n$.)
\end{proposition}
\begin{proof}
 The proof is quite similar to the proof of Lemma 2 in \cite{livshits2009lower}. Suppose that
 \begin{equation}\label{lim-inf-lower-bound-assumption}
  \liminf_{n\rightarrow \infty} \alpha_n \geq \alpha
 \end{equation}
 for some $\alpha \geq 0$.
 
Using the formula \eqref{residual-h-inner-formula}, we see that
\begin{equation}\label{eq-573}
\begin{split}
 \liminf_{n\rightarrow \infty}\frac{-\langle r_{n-1}, h_n\rangle}{q_{n-1}\alpha_n} \geq& \left(\alpha\lim_{n\rightarrow \infty}\frac{}{n}\left[\sum_{k=1}^{n-1}\sum_{j=\max(k+1,K)}^{n-1}\frac{1}{j}\phi\left(\frac{k}{n}\right)\phi\left(\frac{k}{j}\right)\right]\right. \\
 &+ \left.\lim_{n\rightarrow \infty} \frac{1}{n}\left[\frac{K^{-\frac{1}{2} + \beta}}{q_{K-1}\sqrt{K-1}}\sum_{k=1}^{K-1}\phi\left(\frac{k}{n}\right) + \sum_{k=K}^{n-1}\xi_k\phi\left(\frac{k}{n}\right)\right]\right).
 \end{split}
\end{equation}
We proceed to evaluate the limits in the above equation. Using the fact that $\phi$ is Riemann integrable, that $K$ is fixed and that $\xi_k\rightarrow 1$ as $k\rightarrow \infty$, we see that
\begin{equation}
 \lim_{n\rightarrow \infty} \frac{1}{n}\left[\frac{K^{-\frac{1}{2} + \beta}}{q_{K-1}\sqrt{K-1}}\sum_{k=1}^{K-1}\phi\left(\frac{k}{n}\right) + \sum_{k=K}^{n-1}\xi_k\phi\left(\frac{k}{n}\right)\right] = \int_0^1\phi(x)dx.
\end{equation}
Further, applying Lemma \ref{phi-integral-estimates} to $\phi$, we see that 
\begin{equation}
  \sum_{j=\max(k+1,K)}^{n-1}\frac{1}{j}\phi\left(\frac{k}{j}\right) = \int_{\frac{k}{n}}^{\min\left(1,\frac{k}{K}\right)} \phi\left(z\right)\frac{dz}{z} + O(k^{-1}).
\end{equation}
From this, by comparing a Riemann sum with an integral, we see that
\begin{equation}
\begin{split}
 \lim_{n\rightarrow \infty}\frac{1}{n}\left[\sum_{k=1}^{n-1}\sum_{j=\max(k+1,K)}^{n-1}\frac{1}{j}\phi\left(\frac{k}{n}\right)\phi\left(\frac{k}{j}\right)\right] &= \lim_{n\rightarrow \infty}\frac{1}{n}\left[\sum_{k=1}^{n-1}\phi\left(\frac{k}{n}\right)\left[\int_{\frac{k}{n}}^{\min\left(1,\frac{k}{K}\right)} \phi\left(z\right)\frac{dz}{z}\right] + O\left(\sum_{k=1}^{n-1}k^{-1}\right)\right] \\
 & = \lim_{n\rightarrow \infty}\left[\frac{1}{n}\sum_{k=1}^{n-1}\phi\left(\frac{k}{n}\right)\left[\int_{\frac{k}{n}}^{\min\left(1,\frac{k}{K}\right)} \phi\left(z\right)\frac{dz}{z}\right] + O\left(\frac{\log{n}}{n}\right)\right] \\
 & = \lim_{n\rightarrow \infty}\left[\frac{1}{n}\sum_{k=1}^{n-1}\phi\left(\frac{k}{n}\right)\left[\int_{\frac{k}{n}}^{1} \phi\left(z\right)\frac{dz}{z}\right] - \frac{1}{n}\sum_{k=1}^K \phi\left(\frac{k}{n}\right)\int_{\frac{k}{K}}^1\phi(z)\frac{dz}{z}\right]\\
 &= \int_0^1 \phi(x)\left(\int_x^1 \phi(z)\frac{dz}{z}\right)dx,
 \end{split}
\end{equation}
since for large enough $n$, $\phi(k/n) = 0$ for $k=1,...,K$ (recall that $\phi$ is assumed to vanish on $[0,\delta]$ for some $\delta$).

Plugging this into \eqref{eq-573}, we get
\begin{equation}\label{eq-595}
 \liminf_{n\rightarrow \infty}\frac{-\langle r_{n-1}, h_n\rangle}{q_{n-1}\alpha_n} \geq \left(\int_0^1 \phi(x)\left(1+\alpha\int_x^1 \phi(z)\frac{dz}{z}\right)dx\right).
\end{equation}
 Utilizing the asymptotics \eqref{q-asymptotics} and \eqref{inner-product-h-asymptotics}, we obtain that
 \begin{equation}
  \lim_{n\rightarrow \infty}\frac{-\langle r_{n-1}, h_n\rangle}{q_{n-1}} = \frac{\beta}{1 - 2\beta}.
 \end{equation}
 Together with \eqref{eq-595}, this implies that (assuming that \eqref{lim-inf-lower-bound-assumption} holds)
 \begin{equation}\label{eq-673}
  \limsup_{n\rightarrow \infty} \alpha_n \leq \left[\frac{\beta}{1 - 2\beta}\right]\left(\int_0^1 \phi(x)\left(1+\alpha\int_x^1 \phi(z)\frac{dz}{z}\right)dx\right)^{-1}.
 \end{equation}
 In an entirely analogous manner, we prove that $\limsup_{n\rightarrow \infty} \alpha_n \leq \alpha$ (for any $\alpha \geq 0$) implies that
 \begin{equation}\label{eq-677}
  \liminf_{n\rightarrow \infty} \alpha_n \geq \left[\frac{\beta}{1 - 2\beta}\right]\left(\int_0^1 \phi(x)\left(1+\alpha\int_x^1 \phi(z)\frac{dz}{z}\right)dx\right)^{-1}.
 \end{equation}
 We proceed to use the same fixed point argument from \cite{livshits2009lower} to complete the proof. Define $F_{\phi,s}$ by 
 \begin{equation}\label{F-phi-s-definition-761}
  F_{\phi,s}(\alpha) = \left[\frac{\beta}{1 - 2\beta}\right]\left(\int_0^1 \phi(x)\left(1+\alpha\int_x^1 \phi(z)\frac{dz}{z}\right)dx\right)^{-1}.
 \end{equation}
 Proposition \ref{large-K-lemma} implies that $\liminf_{n\rightarrow \infty} \alpha_n \geq 0$. Thus, setting $\pi_0 = 0$ and $\pi_n = F_{\phi,s}(\pi_{n-1})$, we see inductively that \eqref{eq-673} implies that
 \begin{equation}
     \limsup_{m\rightarrow \infty} \alpha_m \leq \pi_{2n-1},
 \end{equation}
 while \eqref{eq-677} implies that
 \begin{equation}
     \liminf_{m\rightarrow \infty} \alpha_m \geq \pi_{2n}.
 \end{equation}
 Thus \eqref{alpha-limit-equation} will be proved if we can show that $\lim_{n\rightarrow \infty} \pi_n = F_{\phi,s}^n(0) = 1$ (here the exponent represents function composition).
 \begin{lemma}\label{recursive-function-lemma}
     The function $F_{\phi,s}^n$ defined in \eqref{F-phi-s-definition-761} satisfies $\lim_{n\rightarrow \infty} F_{\phi,s}^n(0) = 1$.
 \end{lemma}
 \begin{proof}
 Rewrite $F_{\phi,s}$ as 
 \begin{equation}
  F_{\phi,s}(\alpha) = \frac{A}{B + C\alpha},
 \end{equation}
where $A = \frac{\beta}{1 - 2\beta} > 0$, $B = \int_0^1 \phi(x)dx > 0$, and $C = \int_0^1 \phi(x)\left(\int_x^1 \phi(z)\frac{dz}{z}\right)dx > 0$. 

The assumption \eqref{phi-integral-condition-1} implies that $F_{\phi,s}(1) = 1$, i.e., that $A = B+C$. We wish to show that iterating the map $F_{\phi,s}$ converges to this fixed point. This follows from the following simple calculation
 \begin{equation}
  |F^2_{\phi,s}(\alpha) - 1| = \left|\frac{C^2}{C^2 + B^2 + BC\alpha + BC}(\alpha - 1)\right| \leq \frac{C^2}{C^2 + B^2}|\alpha - 1|,
 \end{equation}
 since $\frac{C^2}{C^2 + B^2} < 1$.
 \end{proof}
Lemma \ref{recursive-function-lemma} completes the proof of \eqref{alpha-limit-equation}.

 Next, we prove the asymptotic formula \eqref{asymptotics-coefficients-of-r}. We may assume that $k > K$ and use Lemma \ref{components-of-r-formula} and the fact that $\xi_k \rightarrow 1$ (see Proposition \ref{large-K-lemma}) to get
 \begin{equation}
  \langle r_n, e_k\rangle = -q_n\left(1 + \sum_{j=k+1}^n \langle h_j, e_k\rangle + o_k(1)\right).
 \end{equation}
 Using the definition of $h_n$ \eqref{definition-of-h}, we get
 \begin{equation}
  \langle r_n, e_k\rangle = -q_n\left(1 + \sum_{j=k+1}^n \frac{\alpha_j}{j}\phi\left(\frac{k}{j}\right) + o_k(1)\right) = -q_n\left(1 + \sum_{j=k+1}^n \frac{1+o_k(1)}{j}\phi\left(\frac{k}{j}\right) + o_k(1)\right).
 \end{equation}
 Using Lemma \ref{phi-integral-estimates} and noting that $\int_0^1\frac{\phi(z)}{z}dz < \infty$ (since $\phi = 0$ on $[0,\delta]$), we get
 \begin{equation}
  \langle r_n, e_k\rangle = -q_n\left(1 + \int_\frac{k}{n}^1 \phi(z)\frac{dz}{z} + o_k(1)\right),
 \end{equation}
 as desired.

 Finally, we will prove \eqref{alpha-asymptotics-equation}. For this, we note that \eqref{condition-on-h}, combined with the asymptotics \eqref{inner-product-h-asymptotics} implies that
 \begin{equation}\label{eq-716}
  \frac{\langle r_n, h_{n+1}\rangle}{\langle r_{n-1},h_n\rangle} = 1 + \frac{\beta-1}{n} + o(n^{-1}).
 \end{equation}
 Using the formula \eqref{eq-365} to rewrite this as
 \begin{equation}\label{eq-720}
  \frac{\langle r_n, h_{n+1}\rangle}{\langle r_{n-1},h_n\rangle} = \frac{\langle \frac{q_{n}}{q_{n-1}}r_{n-1} - q_nh_n - q_n\xi_ne_n, h_{n+1}\rangle}{\langle r_{n-1},h_n\rangle}= \frac{q_{n}}{q_{n-1}}\frac{\langle r_{n-1}, h_{n+1}\rangle}{\langle r_{n-1},h_n\rangle} - q_n\frac{\langle h_n + \xi_ne_n, h_{n+1}\rangle}{\langle r_{n-1},h_n\rangle}.
 \end{equation}
 Using the asymptotics \eqref{q-asymptotics} and \eqref{inner-product-h-asymptotics}, we see that
 \begin{equation}
  \frac{q_n}{\langle r_{n-1},h_n\rangle} = c' + O(n^{-1}),
 \end{equation}
 for a constant $c'$. In addition, using the definition of $h_n$, the fact that $\xi_n\rightarrow 1$, and that $\phi$ is smooth, we see that
 \begin{equation}
  \langle h_n + \xi_ne_n, h_{n+1}\rangle = \frac{1}{n+1}\left(\frac{1}{n}\sum_{i=1}^{n-1}\phi\left(\frac{i}{n}\right)\phi\left(\frac{i}{n+1}\right) + \phi\left(\frac{n}{n+1}\right) + o_n(1)\right) = c''n^{-1} + o(n^{-1}),
 \end{equation}
 for a constant $c'' = \phi(1) + \int_0^1\phi^2(x)dx$.
 This means that the last term in \eqref{eq-720} is
 \begin{equation}
  q_n\frac{\langle h_n + \xi_ne_n, h_{n+1}\rangle}{\langle r_{n-1},h_n\rangle} = cn^{-1} + o(n^{-1}),
 \end{equation}
 for a constant $c$.

 Now, we calculate, using the definition of $h_n$,
 \begin{equation}
  \langle r_{n-1},h_n\rangle = \sum_{k=1}^{n-1}\langle r_{n-1},e_k\rangle\langle h_n,e_k\rangle = \frac{\alpha_n}{n}\sum_{k=1}^{n-1}\langle r_{n-1},e_k\rangle \phi\left(\frac{k}{n}\right).
 \end{equation}
 Similarly, we obtain
 \begin{equation}
  \langle r_{n-1},h_{n+1}\rangle = \sum_{k=1}^{n-1}\langle r_{n-1},e_k\rangle\langle h_n,e_k\rangle = \frac{\alpha_{n+1}}{n+1}\sum_{k=1}^{n-1}\langle r_{n-1},e_k\rangle \phi\left(\frac{k}{n+1}\right).
 \end{equation}
 Next, we observe that
 \begin{equation}
  \sum_{k=1}^{n-1}\langle r_{n-1},e_k\rangle \phi\left(\frac{k}{n+1}\right) = \sum_{k=1}^{n-1}\langle r_{n-1},e_k\rangle \phi\left(\frac{k}{n}\right) + \sum_{k=1}^{n-1}\langle r_{n-1},e_k\rangle \left[\phi\left(\frac{k}{n}\right) - \phi\left(\frac{k}{n+1}\right)\right].
 \end{equation}
 Using that $\phi$ is smooth, we see that
 \begin{equation}
  \sum_{k=1}^{n-1}\langle r_{n-1},e_k\rangle \left[\phi\left(\frac{k}{n}\right) - \phi\left(\frac{k}{n+1}\right)\right] = \sum_{k=1}^{n-1}\langle r_{n-1},e_k\rangle \left[\frac{k}{n(n+1)}\phi'\left(\frac{k}{n}\right) + O\left(\frac{k^2}{n^4}\right)\right].
 \end{equation}
 Applying the asymptotics \eqref{asymptotics-coefficients-of-r} and comparing the Riemann sum with an integral, we get
 \begin{equation}
  \sum_{k=1}^{n-1}\langle r_{n-1},e_k\rangle \left[\phi\left(\frac{k}{n}\right) - \phi\left(\frac{k}{n+1}\right)\right] = -q_{n-1}\left(\int_0^1x\phi'(x)\int_x^1\phi(z)\frac{dz}{z}dx + o_n(1)\right).
 \end{equation}
Similarly, we obtain
\begin{equation}
 \sum_{k=1}^{n-1}\langle r_{n-1},e_k\rangle \phi\left(\frac{k}{n}\right) = -q_{n-1}n\left(\int_0^1\phi(x)\int_x^1\phi(z)\frac{dz}{z}dx + o_n(1)\right).
\end{equation}

 Using all of this, we obtain
 \begin{equation}
  \frac{\langle r_{n-1}, h_{n+1}\rangle}{\langle r_{n-1},h_n\rangle} = \frac{\alpha_{n+1}}{\alpha_n}\frac{n}{n+1}\left(1 + n^{-1} \frac{\int_0^1x\phi'(x)\int_x^1\phi(z)\frac{dz}{z}dx + o_n(1)}{\int_0^1\phi(x)\int_x^1\phi(z)\frac{dz}{z}dx + o_n(1)}\right).
 \end{equation}
 Since $\phi\neq 0$ is non-negative, the integral in the denominator above is non-zero and we obtain
 \begin{equation}
  \frac{\langle r_{n-1}, h_{n+1}\rangle}{\langle r_{n-1},h_n\rangle} = \frac{\alpha_{n+1}}{\alpha_n}\frac{n}{n+1}\left(1 + cn^{-1} + o(n^{-1})\right),
 \end{equation}
 for a constant $c$. Using that $\frac{n}{n+1} = 1 - n^{-1} + o(n^{-1})$ and combining the above estimates with \eqref{eq-720} and \eqref{eq-716}, we get
 \begin{equation}
  \frac{\alpha_{n+1}}{\alpha_n} = 1 + \kappa n^{-1} + o(n^{-1}),
 \end{equation}
 for some constant $\kappa$. Finally, since
 \begin{equation}
  \prod_{k=n}^T\frac{\alpha_{k+1}}{\alpha_k} = \frac{\alpha_{T+1}}{\alpha_n}
 \end{equation}
is uniformly bounded in $T$ (since the sequence $\alpha_k$ converges to $1$), we must in fact have $\kappa = 0$, which completes the proof.

\end{proof}

Next, we prove that for sufficiently large $N$, the conditions \eqref{inner-prod-condition} will be satisfied. We consider first the case $k > n$. We have the following result.
\begin{proposition}
 Suppose that $\phi$ satisfies the conditions of Proposition \ref{bar-alpha-lemma} and also the inequality \eqref{integral-intequality-1} from Proposition \ref{first-part-main-proposition}.
 Then there is a sufficiently large $N$ such that for $k > n > N$ we have $|\langle r_{n-1},d_k\rangle| < q_n$.
\end{proposition}
\begin{proof}
 We utilize the formula \eqref{inductive-relation-kgn}. If we can show that for some $\delta > 0$, we have
 \begin{equation}\label{desired-bound}
  \left|\left\langle r_{n-1}, h_k - \frac{\gamma_k}{\gamma_{k-1}}h_{k-1}\right\rangle\right| < (1-\delta)q_nk^{-1},
 \end{equation}
 for every $k > n > N$, then the asymptotics \eqref{inductive-factor-asymptotics} combined with the base case $\langle r_{n-1}, d_n\rangle = q_n$ (which holds by construction) will imply the desired result by induction.
 
 Define $p_k = h_k - \frac{\gamma_k}{\gamma_{k-1}}h_{k-1}$. From the definition of $h_k$, we see that
 \begin{equation}
  \langle p_k, e_m\rangle = \frac{\alpha_k}{k}\phi\left(\frac{m}{k}\right) - \frac{\gamma_k}{\gamma_{k-1}}\frac{\alpha_{k-1}}{k-1}\phi\left(\frac{m}{k-1}\right).
 \end{equation}
 Using \eqref{alpha-asymptotics-equation}, \eqref{alpha-limit-equation}, and that $\phi$ is bounded, we get
 \begin{equation}
  \langle p_k, e_m\rangle = (1+o_k(1))\left(\frac{1}{k}\phi\left(\frac{m}{k}\right) - \frac{\gamma_k}{\gamma_{k-1}}\frac{1}{k-1}\phi\left(\frac{m}{k-1}\right)\right) + o(k^{-2}).
 \end{equation}

 We rewrite the term in parenthesis as
 \begin{equation}
 \left(\frac{1}{k} - \frac{\gamma_k}{\gamma_{k-1}}\frac{1}{k-1}\right)\phi\left(\frac{m}{k}\right) + \frac{\gamma_k}{\gamma_{k-1}}\frac{1}{k-1}\left(\phi\left(\frac{m}{k}\right) - \phi\left(\frac{m}{k-1}\right)\right).
 \end{equation}

 Utilizing the asymptotics \eqref{quotient-asymptotics} and the boundedness of $\phi$, this can be rewritten as
 \begin{equation}
  \frac{(\beta - 1)}{k^2}\phi\left(\frac{m}{k}\right) + \left(1 - \frac{\beta}{k}\right)\frac{1}{k-1}\left(\phi\left(\frac{m}{k}\right) - \phi\left(\frac{m}{k-1}\right)\right) +  o(k^{-2}).
 \end{equation}
Next, we use that $\phi$ is infinitely differentiable and that $\frac{m}{k} - \frac{m}{k-1} = -\frac{m}{k(k-1)}$ to obtain
\begin{equation}
 \phi\left(\frac{m}{k}\right) - \phi\left(\frac{m}{k-1}\right) = -\phi'\left(\frac{m}{k}\right)\frac{m}{k(k-1)} + O\left(\frac{m^2}{k^4}\right).
\end{equation}
Restricting $m \leq n-1 < k$ (since otherwise $\langle r_{n-1}, e_m\rangle = 0$), we get
\begin{equation}
 \langle p_k, e_m\rangle = (1+o_k(1))\left(\frac{(\beta - 1)}{k^2}\phi\left(\frac{m}{k}\right) - \frac{1}{k^2}\phi'\left(\frac{m}{k}\right)\frac{m}{k}\right) +  o(k^{-2}).
\end{equation}
Since $\phi$ is smooth (and thus $\phi'$ is bounded), we get
\begin{equation}
 \langle p_k, e_m\rangle = \frac{(\beta - 1)}{k^2}\phi\left(\frac{m}{k}\right) - \frac{1}{k^2}\phi'\left(\frac{m}{k}\right)\frac{m}{k} +  o(k^{-2}).
\end{equation}

We finally calculate, using the asymptotic formula \eqref{asymptotics-coefficients-of-r} for the components of $r_{n-1}$ to get
\begin{equation}
\begin{split}
 \langle r_{n-1}, p_k\rangle &= \sum_{m=1}^{n-1}\langle r_{n-1}, e_m\rangle\langle p_k, e_m\rangle\\
 &= k^{-1}q_{n-1}\left(\frac{1}{k}\sum_{m=1}^{n-1} \left[\phi'\left(\frac{m}{k}\right)\frac{m}{k} - (\beta - 1)\phi\left(\frac{m}{k}\right) + o_k(1)\right]\left(1 + \int_{\frac{m}{n-1}}^1\phi(z)\frac{dz}{z}+o_m(1)\right)\right).
 \end{split}
\end{equation}
Comparing this Riemann sum with an integral, setting $a = \frac{n-1}{k}$, using that $\phi$ (and thus the integrand) is $C^\infty$, we get
\begin{equation}
 \langle r_{n-1}, p_k\rangle = k^{-1}q_{n-1}\left[\int_0^a(\phi'(x)x - (\beta-1)\phi(x))\left(1 + \int_{a^{-1}x}^1\phi(z)\frac{dz}{z}\right)dx+o_n(1)\right].
\end{equation}
Combined with the assumption \eqref{integral-intequality-1}, this completes the proof, since for sufficiently large $n > N$ the term in brackets above will be strictly less than $1$ in magnitude uniformly in $a$.
\end{proof}

Next, we consider the case where $k < n$. Here we need to choose both $N$ large enough and $\epsilon$ small enough.
\begin{proposition}
 Suppose that $\phi$ satisfies the condition of Proposition \ref{bar-alpha-lemma} and also the inequality \eqref{integral-intequality-2} from Proposition \ref{first-part-main-proposition}
  Then there is a sufficiently large $N$ such that for $n > k \geq N$ we have $|\langle r_{n-1},d_k\rangle| < q_n$. Further, by increasing $N$ if necessary, the denominator in the definition of $\tilde{g}_N$ will be non-zero and we can choose $\epsilon > 0$ small enough so that we also have $|\langle r_{n-1},\tilde{d}_N\rangle| < q_n$ for all $n > N$.
\end{proposition}
\begin{proof}
 We will need to use the formulas \eqref{inner-product-rel-ngk} and \eqref{inner-product-rel-ntildeN}. For this, we will need to estimate
 \begin{equation}
  \sum_{i=k+1}^{n-1} \langle h_i,d_k\rangle.
 \end{equation}
 Utilizing the the inductive definition of $d_k$ \eqref{inductive-definition-x}, we get
 \begin{equation}\label{eq-852}
  \sum_{i=k+1}^{n-1} \langle h_i,d_k\rangle = \gamma_k\sum_{i=k+1}^{n-1} \langle h_i,r_{k-1}\rangle + \sum_{i=k+1}^{n-1} \langle h_i,h_k\rangle + \xi_k\sum_{i=k+1}^{n-1} \langle h_i,e_k\rangle.
 \end{equation}
 We consider the last term above first. From the definition of $h_i$ \eqref{definition-of-h} we get
 \begin{equation}
  \xi_k\sum_{i=k+1}^{n-1} \frac{\alpha_i}{i}\phi\left(\frac{k}{i}\right).
 \end{equation}
 Since $\lim_{k\rightarrow \infty} \xi_k = 1$ by Proposition \ref{large-K-lemma} and $\lim_{k\rightarrow \infty} \alpha_k = 1$ by Proposition \ref{bar-alpha-lemma}, we see that
 \begin{equation}
  \xi_k\sum_{i=k+1}^{n-1} \frac{\alpha_i}{i}\phi\left(\frac{k}{i}\right) = (1 + o_k(1))\sum_{i=k+1}^{n-1} \frac{1}{i}\phi\left(\frac{k}{i}\right).
 \end{equation}
 Applying Lemma \ref{phi-integral-estimates}, we see that
 \begin{equation}
  \xi_k\sum_{i=k+1}^{n-1} \frac{\alpha_i}{i}\phi\left(\frac{k}{i}\right) = \int_{\frac{k}{n}}^1 \phi(z)\frac{dz}{z} + o_k(1).
 \end{equation}
 Setting $a = \frac{k}{n}$ we write this as
 \begin{equation}\label{estimate-third-term}
  \xi_k\sum_{i=k+1}^{n-1} \frac{\alpha_i}{i}\phi\left(\frac{k}{i}\right) = \int_{a}^1 \phi(z)\frac{dz}{z} + o_k(1).
 \end{equation}
 Next, we consider the term
 \begin{equation}
  \sum_{i=k+1}^{n-1} \langle h_i,h_k\rangle = \sum_{i=k+1}^{n-1} \sum_{j=1}^{k-1} \langle h_i, e_j\rangle\langle h_k,e_j\rangle = \sum_{i=k+1}^{n-1} \sum_{j=1}^{k-1} \frac{\alpha_i}{i}\phi\left(\frac{j}{i}\right)\frac{\alpha_k}{k}\phi\left(\frac{j}{k}\right).
 \end{equation}
 Using again \eqref{alpha-limit-equation}, we get
 \begin{equation}
  \sum_{i=k+1}^{n-1} \langle h_i,h_k\rangle = (1 + o_k(1))\sum_{i=k+1}^{n-1} \sum_{j=1}^{k-1} \frac{1}{i}\phi\left(\frac{j}{i}\right)\frac{1}{k}\phi\left(\frac{j}{k}\right),
 \end{equation}
 and Lemma \ref{phi-integral-estimates} implies (recalling the choice $a = \frac{k}{n}$)
 \begin{equation}
 \begin{split}
  \sum_{i=k+1}^{n-1} \langle h_i,h_k\rangle &= (1 + o_k(1))\frac{1}{k}\sum_{j=1}^{k-1}\left[ \phi\left(\frac{j}{k}\right) \int_{\frac{j}{n}}^\frac{j}{k}\phi(z)\frac{dz}{z} + O(j^{-1})\right]\\
  & = (1 + o_k(1))\frac{1}{k}\sum_{j=1}^{k-1} \phi\left(\frac{j}{k}\right) \int_{a\frac{j}{k}}^\frac{j}{k}\phi(z)\frac{dz}{z} + O\left(\frac{\log{k}}{k}\right).
  \end{split}
 \end{equation}
  Comparing a Riemann sum with an integral finally yields
  \begin{equation}\label{estimate-second-term}
   \sum_{i=k+1}^{n-1} \langle h_i,h_k\rangle = \int_0^1 \phi(x)\int_{ax}^x\phi(z)\frac{dz}{z} + o_k(1).
  \end{equation}
 Finally, we consider the first term in \eqref{eq-852}. Using \eqref{asymptotics-coefficients-of-r}, we get
 \begin{equation}
  \gamma_k\sum_{i=k+1}^{n-1} \sum_{j=1}^{k-1} \langle h_i, e_j\rangle\langle r_{k-1},e_j\rangle = -\gamma_kq_{k-1}\sum_{i=k+1}^{n-1} \sum_{j=1}^{k-1} \frac{\alpha_i}{i}\phi\left(\frac{j}{i}\right)\left(1 + \int_{\frac{j}{k}}^1\phi(z)\frac{dz}{z} + o_k(1)\right).
 \end{equation}
 Utilizing the asymptotics \eqref{q-asymptotics}, \eqref{gamma-asymptotics}, and \eqref{alpha-limit-equation}, we get
 \begin{equation}
  \gamma_k\sum_{i=k+1}^{n-1} \sum_{j=1}^{k-1} \langle h_i, e_j\rangle\langle r_{k-1},e_j\rangle = (\beta - 1 + o_k(1))\frac{1}{k}\sum_{i=k+1}^{n-1} \sum_{j=1}^{k-1} \frac{1}{i}\phi\left(\frac{j}{i}\right)\left(1 + \int_{\frac{j}{k}}^1\phi(z)\frac{dz}{z} + o_k(1)\right).
 \end{equation}
 Proceeding as before, using Lemma \ref{phi-integral-estimates}, recalling the choice $a = \frac{k}{n}$, and comparing a Riemann sum with an integral, we finally obtain
 \begin{equation}\label{estimate-first-term}
  \gamma_k\sum_{i=k+1}^{n-1} \sum_{j=1}^{k-1} \langle h_i, e_j\rangle\langle r_{k-1},e_j\rangle = (\beta - 1)\int_0^1\left(\int_{ax}^x\phi(z)\frac{dz}{z}\right)\left(1 + \int_{x}^1\phi(z)\frac{dz}{z}\right)dx + o_k(1).
 \end{equation}
 Combining the estimates \eqref{estimate-first-term}, \eqref{estimate-second-term}, and \eqref{estimate-third-term}, we get
 \begin{equation}\label{final-estimate}
  \sum_{i=k+1}^{n-1} \langle h_i,d_k\rangle = \int_0^1\left((\beta - 1) + (\beta - 1)\int_{x}^1\phi(z)\frac{dz}{z} + \phi(x)\right)\left(\int_{ax}^x\phi(z)\frac{dz}{z}\right) dx + \int_a^1\phi(x)\frac{dx}{x} + o_k(1).
 \end{equation}
 Our assumption \eqref{integral-intequality-2} on $\phi$ now imply that for sufficiently large $k$, the sum $\sum_{i=k+1}^{n-1} \langle h_i,d_k\rangle$ will satisfy the bounds
 \begin{equation}\label{final-sum-bound-915}
  -1 < \sum_{i=k+1}^{n-1} \langle h_i,d_k\rangle < 1,
 \end{equation}
 uniformly in $n$ and $k$. We now use equation \eqref{inner-product-rel-ngk} to see that
 \begin{equation}
  |\langle r_{n-1},d_k\rangle| = q_{n-1}\left|\sum_{i=k+1}^{n-1} \langle h_i,d_k\rangle\right|\leq q_n(1+O(n^{-1}))\left|\sum_{i=k+1}^{n-1} \langle h_i,d_k\rangle\right|.
 \end{equation}
 The bound \eqref{final-sum-bound-915} implies that for sufficiently large $N$ the final absolute value above is strictly smaller than $1$ uniformly in $n > k \geq N$. Thus, by choosing $N$ potentially larger, we can guarantee that for all $n > k\geq N$,
 \begin{equation}\label{bound-condition-on-N-964}
(1+O(n^{-1}))\left|\sum_{i=k+1}^{n-1} \langle h_i,d_k\rangle\right| < 1,
 \end{equation}
 as well. Fix $N$ large enough to guarantee this. Then we get
 \begin{equation}
  |\langle r_{n-1},d_k\rangle| < q_n,
 \end{equation}
 as desired.
 
 Finally, we bound $\langle r_{n-1}, \tilde{d}_N\rangle$. Consider the inner product formula \eqref{inner-product-rel-ntildeN}. Utilizing the definition of $\tilde{g}_N$, we obtain
 \begin{equation}
  \sum_{i=N+1}^{n-1} \langle h_i,\tilde{g}_N\rangle = C_\epsilon \sum_{i=N+1}^{n-1} \langle h_i,d_N\rangle + \frac{\epsilon}{\|r_N\|} \sum_{i=N+1}^{n-1} \langle h_i,r_N\rangle,
 \end{equation}
where $$C_\epsilon = \sqrt{1-\epsilon^2}.$$
We note that since $r_N\in \text{span}(e_1,...,e_N)$, we get
\begin{equation}
 \sum_{i=N+1}^{n-1} \langle h_i,r_N\rangle = \sum_{i=N+1}^{n-1} \sum_{j=1}^N \langle h_i,e_j\rangle\langle r_N,e_j\rangle = \sum_{j=1}^N \langle r_N,e_j\rangle \sum_{i=N+1}^{n-1}\frac{\alpha_i}{i}\phi\left(\frac{j}{i}\right).
\end{equation}
Since $\phi$ is assumed to vanish on $[0,\delta]$ we get that 
\begin{equation}
 \sum_{i=N+1}^{n-1} \langle h_i,r_N\rangle \leq \sum_{j=1}^N \langle r_N,e_j\rangle \sum_{i=N+1}^{\delta^{-1}N}\frac{\alpha_i}{i}\phi\left(\frac{j}{i}\right) < \infty,
\end{equation}
and so this sum is bounded uniformly in $n$ (but depending upon $N$).

Plugging these bounds into \eqref{inner-product-rel-ntildeN} we see that
\begin{equation}\label{epsilon-estimate}
 |\langle r_{n-1},\tilde{g}_N\rangle| \leq q_{n}(1+O(n^{-1}))\left|C_\epsilon \sum_{i=k+1}^{n-1} \langle h_i,d_N\rangle+ K\epsilon\right|,
\end{equation}
for a constant $K = K(N)$ depending upon $N$. Utilizing the bound \eqref{bound-condition-on-N-964}, we see that for sufficiently small $\epsilon > 0$ (depending upon $K$ and thus $N$), we finally get (since $C_\epsilon < 1$)
\begin{equation}
 |\langle r_{n-1},\tilde{g}_N\rangle| < q_n,
\end{equation}
for all $n > N$. This completes the proof.
\end{proof}

\section{Proof of Proposition \ref{second-part-main-proposition}}\label{second-part-main-proposition-proof-section}
In this Section, we prove Proposition \ref{second-part-main-proposition}, which guarantees the existence of a function $\phi$ satisfying the conditions of Proposition \ref{first-part-main-proposition}.

To construct such a function $\phi$, let $\nu\geq 0$ be a $C^\infty$ bump function which is supported on $[-1,1]$ and satisfies 
$$\int_{-\infty}^\infty\nu(x)dx = 1.$$
Let $t > 0$ and write $\nu_t = t^{-1}\nu\left(t^{-1}x\right)$.

Fix a parameter $\tau\in (0,1)$ to be determined later. For a continuously differentiable function $f:[\tau,1]\rightarrow \mathbb{R}_+$, we extend $f$ to the whole real line by setting $f(x) = 0$ for $x < \tau$ and $f(x) = f(1)$ for $x > 1$. Then, we smooth $f$ using the bump function $\nu$ and normalize so that \eqref{phi-integral-condition-1} is satisfied. Specifically, we set
\begin{equation}
\bar{\phi}_t(x) = \int_{-\infty}^\infty f(x - y)\nu_t(y)dy,
\end{equation}
and $\phi(x) = C_t\bar{\phi}_t$, where the constant $C_{t}$ is chosen to satisfy \eqref{phi-integral-condition-1}. 

Since $\bar{\phi}_t \rightarrow f$ in $L^1$ as $t\rightarrow 0$ and both functions vanish in a neighborhood of $0$ for $t < \tau$, we see that
\begin{equation}
 \lim_{t\rightarrow 0} \left(\int_0^1 \bar{\phi}_t(x)\left(1+\int_x^1 \bar{\phi}_t(z)\frac{dz}{z}\right)dx\right) = \left(\int_\tau^1 f(x)\left(1+\int_x^1 f(z)\frac{dz}{z}\right)dx\right).
\end{equation}
This means that if $f$ satisfies
\begin{equation}\label{f-integral-condition-1}
 \left(\int_\tau^1 f(x)\left(1+\int_x^1 f(z)\frac{dz}{z}\right)dx\right) = \frac{\beta}{1 - 2\beta},
\end{equation}
then the normalizing constant will satisfy $C_t\rightarrow 1$ as $t\rightarrow 0$.

Using this, we calculate that if $f$ satisfies
\begin{equation}\label{f-integral-intequality-1}
 \sup_{a\in[\tau,1]}\left|f(\tau)\tau \left(1 + \int_{a^{-1}\tau}^1f(z)\frac{dz}{z}\right) + \int_\tau^a(f'(x)x - (\beta - 1)f(x)) \left(1 + \int_{a^{-1}x}^1f(z)\frac{dz}{z}\right)dx\right| < 1
\end{equation}
and
\begin{equation}\label{f-integral-intequality-2}
  \sup_{a\in[\tau,1]}\left|\int_\tau^1\left((\beta - 1) + (\beta - 1)\int_{x}^1f(z)\frac{dz}{z} + f(x)\right)\left(\int_{ax}^xf(z)\frac{dz}{z}\right) dx + \int_a^1f(x)\frac{dx}{x}\right| < 1,
 \end{equation}
in addition to \eqref{f-integral-condition-1}, then for sufficiently small $t > 0$, $\phi$ will satisfy \eqref{phi-integral-condition-1}, \eqref{integral-intequality-1}, and \eqref{integral-intequality-2}. 

Indeed, for any $t < \tau$, we clearly have $\phi\in C^\infty$ and $\phi(x) = 0$ on $[0,\delta]$ if $\delta < \tau - t$. Now, since $\phi$ vanishes in a neighborhood of $0$, we see that the integral in \eqref{integral-intequality-2} converges to the corresponding integral for $f$ uniformly in $a$. For the condition \eqref{integral-intequality-1}, we note that
\begin{equation}
    1 + \int_{x}^1\phi_t(z)\frac{dz}{z} \rightarrow 1 + \int_{x}^1f(z)\frac{dz}{z},
\end{equation}
as $t\rightarrow 0$ uniformly in $x$. Further, $\phi_t$ converges to $f$ in $L^1$, so that the only problematic term arises from the $x\phi_t'(x)$ term in \eqref{integral-intequality-1}. However, $\phi_t'(x)$ converges uniformly to $f'(x)$ for $x$ outside of $[\tau - t,\tau + t]$. For this term, we note that
\begin{equation}
    x\left(1 + \int_{a^{-1}x}^1\phi_t(z)\frac{dz}{z}\right)
\end{equation}
is uniformly continuous in $x,t$ and $a \geq \tau - t$ (for smaller values of $a$ the entire integral in \eqref{integral-intequality-1} vanishes). On the interval $[\tau - t,\tau + t]$ the function $\phi_t$ increases rapidly from $0$ to $f(\tau)$, so the contribution from the integral of this term over $[\tau - t,\tau + t]$ converges to
\begin{equation}\label{eq-895}
    f(\tau)\tau \left(1 + \int_{a^{-1}\tau}^1f(z)\frac{dz}{z}\right),
\end{equation}
uniformly in $a \geq \tau + t$ as $t\rightarrow 0$. For $a\in [\tau - t,\tau + t]$, the integral in \eqref{integral-intequality-1} reduces to
\begin{equation}
    \int_{\tau - t}^a x\phi_t'(x)\left(1 + \int_{a^{-1}x}^1 \phi_t(z)\frac{dz}{z}\right)dx - (\beta - 1)\int_{\tau - t}^a\phi_t(x)\left(1 + \int_{a^{-1}x}^1 \phi_t(z)\frac{dz}{z}\right)dx.
\end{equation}
Since $\phi_t$ is bounded and $a\in [\tau - t,\tau + t]$ we see that this converges to
\begin{equation}
    \tau\int_{\tau - t}^a \phi_t'(x)dx = \tau\phi_t(a),
\end{equation}
as $t\rightarrow 0$ uniformly in $a\in [\tau - t,\tau + t]$. Since $|\tau f(\tau)| < 1$ (setting $a = \tau$ in \eqref{f-integral-intequality-1}), for sufficiently small $t$ we will have $|\tau\phi_t(a)| \leq \tau f(\tau) + \epsilon < 1$ for all $a\in [\tau - t,\tau + t]$ (as $\phi_t$ increases to $\phi_t(\tau+t)$ on the interval $[\tau - t,\tau + t]$ and $|\phi_t(\tau+t) - f(\tau)|\rightarrow 0$). Thus $\phi$ satisfies all of the conditions of Proposition \ref{first-part-main-proposition}. The task now is to choose $\tau$ appropriately and to find such an $f$.

 The key to this is to introduce the function
\begin{equation}\label{definition-of-int-F}
 F(a) = \int_\tau^af(x)\left(1+\int_{a^{-1}x}^1 f(z)\frac{dz}{z}\right)dx,
\end{equation}
and to rewrite the conditions on $f$ in terms of $F$. The condition \eqref{f-integral-condition-1} then becomes
\begin{equation}\label{F-endpoint-condition}
 F(1) = \frac{\beta}{(1 - 2\beta)}.
\end{equation}
Integrating by parts to remove the $f'$ term in \eqref{f-integral-intequality-1} and observing that
\begin{equation}
 F'(a) = f(a) + a^{-1}\int_{\tau}^a f(x)f(a^{-1}x)dx,
\end{equation}
we obtain
\begin{equation}
 \sup_{a\in[\tau,1]}|aF'(a) - \beta F(a)| < 1.
\end{equation}
To determine the function $F$, we choose to set (note that this is also what is needed to make Sil'nichenko's analysis tight) 
\begin{equation}
 aF'(a) - \beta F(a) = c,
\end{equation}
for a positive constant $c < 1$. Combined with the definition \eqref{definition-of-int-F}, which implies that $F(\tau) = 0$, and the endpoint condition \eqref{F-endpoint-condition}, this forces
\begin{equation}\label{equation-for-c}
 c = \frac{\beta^2}{(1-2\beta)(\tau^{-\beta} - 1)},
\end{equation}
and
\begin{equation}\label{form-of-F}
 F(a) = \frac{c}{\beta}\left(\tau^{-\beta}a^\beta - 1\right).
\end{equation}
Since $c < 1$ must hold, we obtain the following condition on $\beta$ and $\tau$
\begin{equation}\label{parameter-conditions-1}
 \frac{\beta^2}{(1-2\beta)(\tau^{-\beta} - 1)} < 1.
\end{equation}
Differentiating the integrand in \eqref{f-integral-intequality-2} with respect to $a$, rewriting in terms of $F$, and noting that for $a = 1$ the integrals vanish, we see that the condition \eqref{f-integral-intequality-2} becomes
\begin{equation}
 \sup_{b\in [\tau,1]} \left|\int_b^1 a^{-2}(c - (1-2\beta)F(a))da\right| < 1.
\end{equation}
Plugging \eqref{form-of-F} into these conditions and noting that the integral above is minimized in $b$ when the integrand is $0$, which occurs for $$b = \tau\left(\frac{1-\beta}{1-2\beta}\right)^{\frac{1}{\beta}},$$ and is maximized when $b = \tau$, we obtain the following two conditions on $\beta$, and $\tau$
\begin{equation}\label{parameter-conditions-2}
 \frac{\beta}{(1-2\beta)(\tau^{-\beta} - 1)}\left(\beta - 1 + \frac{\tau^{-\beta}(1-2\beta)}{1-\beta} -\beta\tau^{-1}\left(\frac{1-2\beta}{1-\beta}\right)^{\frac{1}{\beta}}\right) > -1,
\end{equation}
\begin{equation}\label{parameter-conditions-3}
 \frac{\beta}{(1-2\beta)(\tau^{-\beta} - 1)}\left(\beta - 1 + \frac{\tau^{-\beta}(1-2\beta)}{1 - \beta}+\frac{\tau^{-1}\beta^2}{(1-\beta)}\right) < 1.
\end{equation}
A routine calculation shows that condition \eqref{parameter-conditions-1} means that $\tau$ must be chosen so that
\begin{equation}\label{eq-1101}
 \frac{(1-\beta)^2}{1-2\beta} < \tau^{-\beta}.
\end{equation}
Let $\tau^*$ be the value of $\tau$ which gives equality in the above bound. By choosing $\tau < \tau^*$, we see that it suffices for the inequalities \eqref{parameter-conditions-2} and \eqref{parameter-conditions-3} to hold (strictly, of course) for $\tau = \tau^*$. Plugging in this value and performing a routine, yet tedious, calculation, we see that \eqref{parameter-conditions-2} becomes
\begin{equation}\label{eq-1105}
 \beta(1-\beta)^{\frac{1}{\beta}} < 1,
\end{equation}
and \eqref{parameter-conditions-3} becomes
\begin{equation}\label{eq-1109}
 \left(\frac{\beta}{1-\beta}\right)^\beta\left(\frac{(1-\beta)^2}{1-2\beta}\right) < 1.
\end{equation}
We see that \eqref{eq-1105} holds for any $\beta\in (0,\frac{1}{2})$ since the left hand side is $ < 1$ and the right hand side is $\geq 1$. Thus the only condition on $\beta$ which must be satisfied is \eqref{eq-1109}, which is exactly the relation \eqref{beta-condition}.

The only step left in the proof of Theorem \ref{main-theorem} is to show that the equation \eqref{definition-of-int-F} can be solved for $f$ when $F$ is given by \eqref{form-of-F}. Differentiating \eqref{definition-of-int-F} with respect to $a$, we obtain the following integral differential equation on $[\tau,1]$:
\begin{equation}\label{differential-integral-equation}
 f(a) + a^{-1}\int_{\tau}^a f(x)f(a^{-1}x)dx = G(a) := F'(a) = c\tau^{-\beta}a^{\beta - 1}.
\end{equation}
We will use the following Lemma, based on the Schauder fixed point theorem, to solve this equation.
\begin{lemma}\label{schauder-lemma}
 Suppose that $G(x)$ is continuously differentiable and satisfies the bound
 \begin{equation}\label{definition-of-rg}
  R_G := \sup_{a\in [\tau,1]}a^{-2}\int_{\tau}^a G(a^{-1}x)dx < 1.
 \end{equation}
 In addition, assume that there exist functions $0\leq g_1(x)\leq g_2(x)\leq G(x)$ on $[\tau,1]$ such that
 \begin{equation}
  G(a) - a^{-1}\int_{\tau}^a g_1(x)g_1(a^{-1}x)dx \leq g_2(a),
 \end{equation}
 and
 \begin{equation}
  G(a) - a^{-1}\int_{\tau}^a g_2(x)g_2(a^{-1}x)dx \geq g_1(a).
 \end{equation}
 Then the equation \eqref{differential-integral-equation} has a continuously differentiable solution $f \geq 0$.
\end{lemma}
\begin{proof}
Consider the map $T_G:C([\tau,1])\rightarrow C([\tau,1])$ defined by
\begin{equation}
 T_G(f)(a) = G(a) - a^{-1}\int_{\tau}^a f(x)f(a^{-1}x)dx.
\end{equation}
For a fixed $M$ to be determined later, define the set
\begin{equation}
 S_M := \{f\in C^1([\tau,1]),~g_1(x) \leq f(x) \leq g_2(x),~|f'(x)| \leq M\}\subset C([\tau,1]).
\end{equation}
By the Arzela-Ascoli theorem, $S_M$ is a compact subset of $C([\tau,1])$. We will show that for sufficiently large $M$, $T_G$ maps $S_M$ into itself. Indeed, if $g_1(x) \leq f(x) \leq g_2(x)$, then we clearly have (since $g_1,g_2$ and thus $f$ are non-negative)
\begin{equation}
 T_G(f)(a) = G(a) - a^{-1}\int_{\tau}^a f(x)f(a^{-1}x)dx \leq G(a) - a^{-1}\int_{\tau}^a g_1(x)g_1(a^{-1}x)dx \leq g_2(a)
\end{equation}
and
\begin{equation}
 T_G(f)(a) = G(a) - a^{-1}\int_{\tau}^a f(x)f(a^{-1}x)dx \geq G(a) - a^{-1}\int_{\tau}^a g_2(x)g_2(a^{-1}x)dx \geq g_1(a),
\end{equation}
by assumption on $g_1$ and $g_2$. Further, taking derivates, we see that
\begin{equation}
 T_G(f)'(a) = G'(a) + a^{-2}\int_{\tau}^a f(x)f(a^{-1}x)dx - a^{-1}f(a)f(1) + a^{-3}\int_{\tau}^a f(x)f'(a^{-1}x)dx.
\end{equation}
We now integrate the last integral above by parts to obtain
\begin{equation}
\begin{split}
 a^{-3}\int_{\tau}^a f(x)f'(a^{-1}x)dx &= a^{-2}\int_{\tau}^a f(x)\frac{d}{dx}f(a^{-1}x)dx \\
 &= a^{-2}(f(a)f(1) - f(\tau)f(a^{-1}\tau)) - a^{-2}\int_{\tau}^a f'(x)f(a^{-1}x)dx.
 \end{split}
\end{equation}
Since $G$ is fixed, $f\leq g_2\leq G$ and $|f'(x)|\leq M$ by assumption, collecting the preceding equations gives a bound of
\begin{equation}
 |T_G(f)'(a)| \leq K + R_GM,
\end{equation}
where $K$ is a constant only depending upon $G$ and not upon the bound $M$. If $R_G < 1$, we can choose $M > K(1-R_G)^{-1}$ to guarantee that $|T_G(f)'(a)| \leq M$ as well. This proves that $T_G$ maps $S_M$ to itself. The proof is now completed by invoking the Schauder fixed point theorem.
\end{proof}
Finally, we verify the assumptions of Lemma \ref{schauder-lemma} for the function $G$ in \eqref{differential-integral-equation}. A straightforward calculation gives that the integral in \eqref{definition-of-rg} is maximized when $a = \min(1,\tau(1+\beta)^{1/\beta})$ and its value is given by
\begin{equation}
 \frac{ca^{-(\beta+1)}}{\beta}\left(\left(\frac{a}{\tau}\right)^\beta  -1\right).
\end{equation}
Plugging in the values $\beta = \beta^*$ and $\tau = \tau^*$ which give equality in equations \eqref{eq-1109} and \eqref{eq-1101} for $s=1$, a straightforward numerical calculation gives that $R_G < 1$. Since this value is continuous in $\beta$ and $\tau$, the same holds for close enough $\beta$ and $\tau$. 

Finally, we verify the existence of the functions $g_1$ and $g_2$ in Lemma \ref{schauder-lemma}. We consider the modified map
\begin{equation}
 \tilde{T}_G(f)(a) = \max\left(0,G(a) - a^{-1}\int_{\tau}^a f(x)f(a^{-1}x)dx\right).
\end{equation}
Starting with the function $f_0 = G$, we iterate to obtain the sequence $f_k = \tilde{T}_G(f_{k-1})$. The map $\tilde{T}_G$ is monotone, so that $f_0 \geq f_2\geq f_4\cdots$ and $f_1\leq f_3\leq f_5\cdots$. Consequently, if $f_{2n+1} > 0$ for some $n$, then $g_1 = f_{2n+1}$ and $g_2 = f_{2n}$ will satisfy the conditions of Lemma \ref{schauder-lemma}. This follows since by the monotonicity of $\tilde{T}_G$, $f \leq f_{2n}$ implies that $\tilde{T}_G(f) \geq f_{2n+1}$ and $f \geq f_{2n+1}$ implies that $\tilde{T}_G(f) \leq f_{2n+2} \leq f_{2n}$. Finally, $\tilde{T}_G(f) \geq f_{2n+1} > 0$ implies that $\tilde{T}_G$ and $T_G$ coincide, which shows that the conditions of the Lemma are satisfied.

We verify numerically for $G(x) = c\tau^{-\beta}x^{\beta - 1}$ given in \eqref{differential-integral-equation}, with $\beta = \beta^*$ and $\tau = \tau^*$ (the optimal values for $s=1$), that the iterates satisfy $f_3 > 0$. For reference, a plot of $f_0,f_1,f_2$, and $f_3$ is shown in Figure \ref{numerical-figure}. Since $f_3$ depends continuously on $G$ (and thus on $\beta$ and $\tau$), this completes the proof of Proposition \ref{second-part-main-proposition}. Finally, in Figure \ref{numerical-figure} we solve equation \eqref{differential-integral-equation} numerically to give an idea of what the optimal $f$ looks like.

\begin{figure}[H]
\begin{center}
 \includegraphics[scale=0.41]{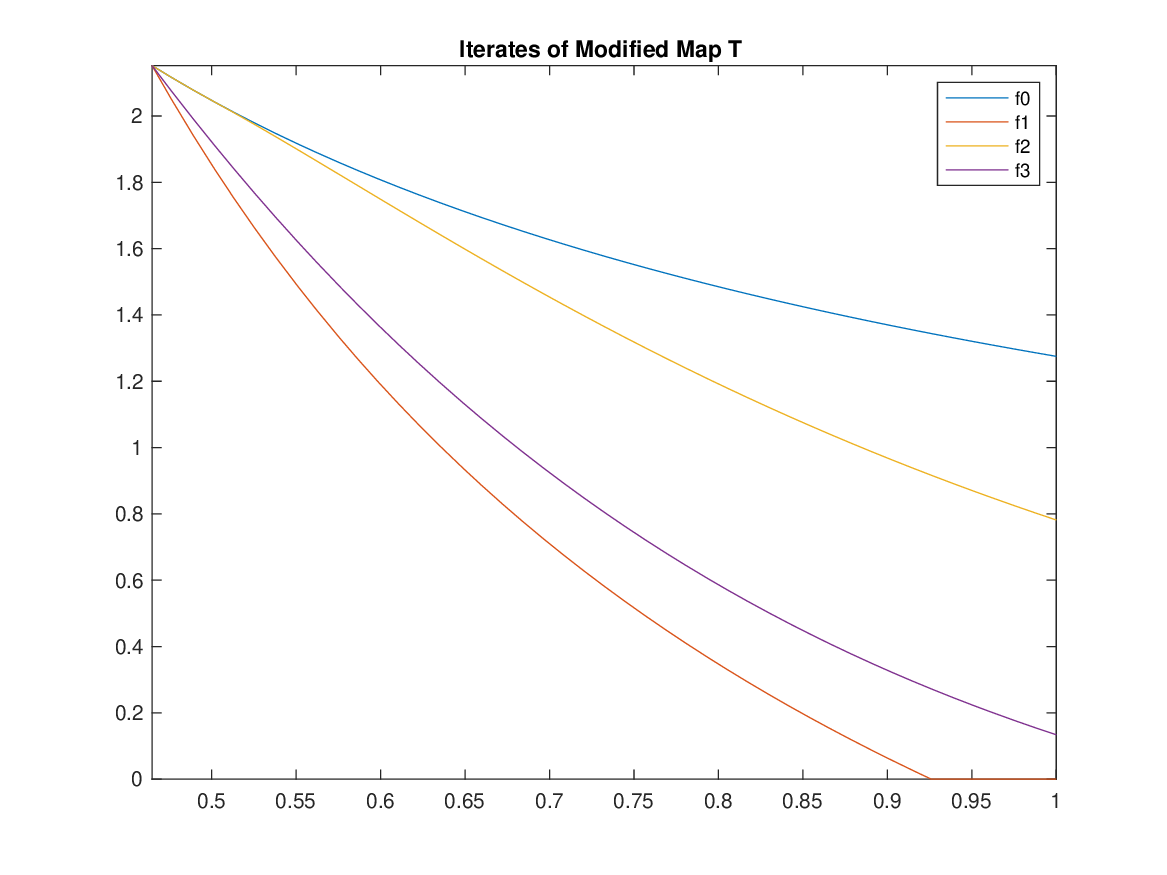}
 \includegraphics[scale=0.41]{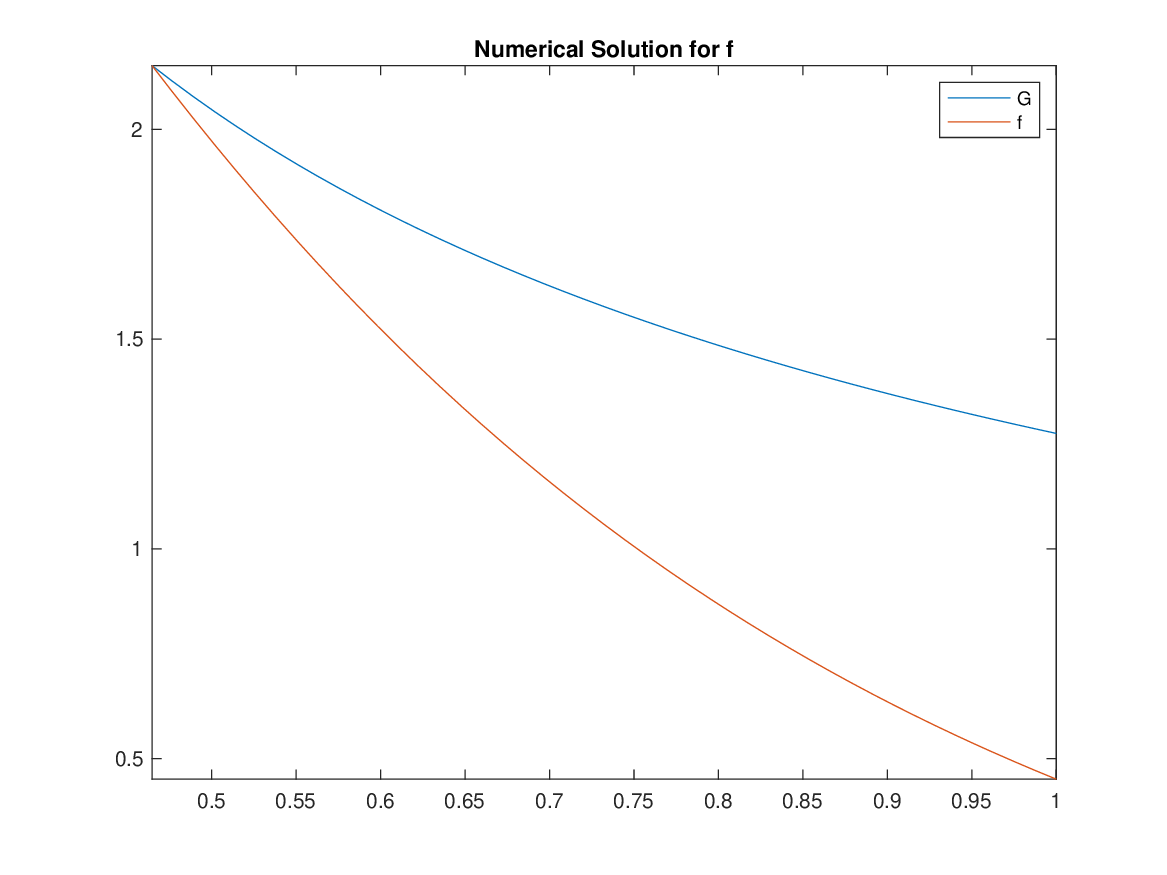}
 \end{center}
\caption{\textbf{Left}: The first $3$ iterates of $\tilde{T}_G$, which demonstrate that $f_3 > 0$. \textbf{Right}: The numerically calculated solution to \eqref{differential-integral-equation} for the $G$ corresponding to $\beta^*$ and $\tau^*$ for $s=1$.}
\label{numerical-figure}
\end{figure}

\section{Acknowledgements}
We would like to thank Andrew Barron, Ron DeVore, Jinchao Xu, Vladimir Temlyakov, and Matias Cattaneo for helpful discussions. JWS was supported in part by the National Science Foundation through DMS-2424305 and CCF-2205004 as well as the Office of Naval Research MURI N00014-20-1-2787. JMK was supported in part by the National Science Foundation through CAREER DMS-2239448, DMS-2054808, and HDR TRIPODS CCF-1934924.

\bibliographystyle{spmpsci}
\bibliography{refs}

\end{document}